\documentclass{article}

\usepackage{arxiv}

\usepackage[utf8]{inputenc} 
\usepackage[T1]{fontenc}    
\usepackage{hyperref}       
\usepackage{url}            
\usepackage{booktabs}       
\usepackage{amsfonts}       
\usepackage{nicefrac}       
\usepackage{microtype}      
\usepackage{lipsum}
\usepackage{graphicx}
\usepackage{lineno}

\usepackage{lineno}
\usepackage{times}
\usepackage{graphicx}
\usepackage{amsmath,amssymb,amsthm}
\usepackage{booktabs}
\usepackage{multirow}
\usepackage{adjustbox}
\usepackage{pifont}
\usepackage{makecell}
\usepackage[numbers,sort&compress]{natbib}
\usepackage[table]{xcolor}
\definecolor{SoftGray}{gray}{0.93}
\definecolor{SoftBlue}{rgb}{0.90, 0.95, 1.0}
\definecolor{SoftGreen}{rgb}{0.90, 1.0, 0.90}
\definecolor{SoftLime}{rgb}{0.95, 1.0, 0.90}

\usepackage{algorithm}
\usepackage{algpseudocode}
\usepackage{listings}
\usepackage{amsmath,amsfonts,bm}
\usepackage{mathtools}  
\usepackage{microtype}

\usepackage{amsthm}

\newcommand{\softmax}{\mathrm{softmax}}

\newcommand{\eqfit}[1]{%
  \begin{adjustbox}{max width=\linewidth,center}%
  $\displaystyle #1$%
  \end{adjustbox}%
}
\newcommand{\eqfitalign}[1]{%
  \begin{adjustbox}{max width=\linewidth,center}%
  $\begin{aligned}#1\end{aligned}$%
  \end{adjustbox}%
}

\def\eqref#1{equation~\ref{#1}}

\def\1{\bm{1}}

\DeclareMathAlphabet{\mathsfit}{\encodingdefault}{\sfdefault}{m}{sl}
\SetMathAlphabet{\mathsfit}{bold}{\encodingdefault}{\sfdefault}{bx}{n}

\newtheorem{definition}{Definition}
\newtheorem{remark}{Remark}
\newtheorem{lemma}{Lemma}

\newtheorem{theorem}{Theorem}
\newtheorem{corollary}{Corollary}
\newtheorem{proposition}{Proposition}
\newtheorem{assumption}{Assumption}
\renewcommand{\eqref}[1]{(\ref{#1})}

\definecolor{SoftmaxRed}{HTML}{EB4537}   
\definecolor{HedgehogYellow}{HTML}{FAC230} 
\definecolor{LunaBlue}{HTML}{4286F3}     
\definecolor{PlaceGrey}{RGB}{230,230,230}

\usepackage{xcolor,graphicx}
\definecolor{SoftmaxRed}{HTML}{EB4537}
\definecolor{HedgehogYellow}{HTML}{FAC230}
\definecolor{LunaBlue}{HTML}{4286F3}
\definecolor{PlaceGrey}{RGB}{230,230,230}

\title{LUNA: Linear Universal Neural Attention with Generalization Guarantees}

\author{
\textbf{Ashkan Shahbazi}$^{\ast,1}$\qquad
\textbf{Ping He}$^{\ast,1}$\qquad
\textbf{Ali Abbasi}$^{1}$\\[0.3em]
\textbf{Yikun Bai}$^{1}$\qquad
\textbf{Xinran Liu}$^{1}$\qquad
\textbf{Elaheh Akbari}$^{1}$\\[0.3em]
\textbf{Darian Salehi}$^{3}$\qquad
\textbf{Navid NaderiAlizadeh}$^{4}$\qquad
\textbf{Soheil Kolouri}$^{1,2}$\\[0.8em]
$^{1}$Department of Computer Science, Vanderbilt University, Nashville, TN, USA\\[0.2em]
$^{2}$Department of Electrical \& Computer Engineering, Vanderbilt University, Nashville, TN, USA\\[0.2em]
$^{3}$Department of Computer Science, Duke University, Durham, NC, USA\\[0.2em]
$^{4}$Department of Biostatistics \& Bioinformatics, Duke University, Durham, NC, USA\\[0.5em]
$^{\ast}$Equal contribution
}

\begin{document}
\maketitle
\begin{abstract}
Scaling attention faces a critical bottleneck: the $\mathcal{O}(n^2)$ quadratic computational cost of softmax attention, which limits its application in long-sequence domains. While linear attention mechanisms reduce this cost to $\mathcal{O}(n)$, they typically rely on fixed random feature maps, such as random Fourier features or hand-crafted functions. This reliance on static, data-agnostic kernels creates a fundamental trade-off, forcing practitioners to sacrifice significant model accuracy for computational efficiency. We introduce \textsc{LUNA}, a kernelized linear attention mechanism that eliminates this trade-off, retaining linear cost while matching and surpassing the accuracy of quadratic attention. \textsc{LUNA} is built on the key insight that the kernel feature map itself should be learned rather than fixed a priori. By parameterizing the kernel, \textsc{LUNA} learns a feature basis tailored to the specific data and task, overcoming the expressive limitations of fixed-feature methods. \textsc{Luna} implements this with a learnable feature map that induces a positive-definite kernel and admits a streaming form, yielding linear time and memory scaling in the sequence length. Empirical evaluations validate our approach across diverse settings. On the Long Range Arena (LRA), \textsc{Luna} achieves state-of-the-art average accuracy among efficient Transformers under compute parity, using the same parameter count, training steps, and approximate FLOPs. \textsc{Luna} also excels at post-hoc conversion: replacing softmax in fine-tuned BERT and ViT-B/16 checkpoints and briefly fine-tuning recovers most of the original performance, substantially outperforming fixed linearizations.
\end{abstract}


\section{Introduction}
\label{sec:intro}


Transformers \citep{vaswani2017attention} underpin state-of-the-art systems across language \citep{grattafiori2024llama}, vision \citep{khan2022transformers}, audio \citep{gulati2020conformer}, multi-modal learning \citep{liu2023visual}, and scientific domains \citep{jumper2021highly,dalla2025nucleotide}. Their core mechanism, attention, models long-range token dependencies but incurs quadratic cost in the sequence length, which limits context scaling. This has motivated a large literature on \emph{linear attention}, which reduces complexity via structured sparsity, low-rank compression, or kernel feature expansions \citep{katharopoulos2020transformers,wang2020linformer,choromanski2020rethinking,chen2021skyformer,xiong2021nystromformer,meng2025polaformer,shen2021efficientattention}. However, these architectures typically commit to a fixed kernel or feature map, whether derived from the softmax exponential kernel or hand-crafted nonlinearities, and thus cannot adapt their inductive bias to the statistics of a given task or dataset. Our work follows this line but asks a different question: rather than fixing the kernel feature map \emph{a priori} (random or engineered), can we \emph{learn} the feature family directly from data—preserving linear complexity while tailoring the kernel to the task?

Learning such kernels while preserving the linear-attention regime is non-trivial. The streaming formulations that enable linear time and memory typically rely on rigid algebraic structure: the kernel must admit a non-negative feature representation and a stable decomposition into key- and query-side statistics. Naively parameterizing the feature map can break positive-definiteness, destroy the streaming factorization, or lead to brittle training dynamics. Instead, we seek a learnable kernel family that preserves the linear computation while exposing enough flexibility to adjust its inductive bias across architectures and domains.

We address this challenge by introducing \textsc{Luna}, a linear-time attention mechanism that replaces hand-crafted random features with a fully learnable kernel feature family. Concretely, \textsc{Luna} parameterizes (i) input projection matrices that capture distributional structure and (ii) a bank of channel functions with a token-wise envelope that together define the kernel nonlinearity. Attention is computed via the standard kernelized factorization in linear time, making \textsc{Luna} a drop-in replacement for softmax or prior linear modules. We train all components end-to-end on each task using the task loss. This separation of representation (the learned nonlinearity) and mixing (the projections) yields expressive efficiency: the model retains linear-time complexity in the sequence length while adapting its feature basis to data rather than sampling from a fixed spectral measure.

Beyond training from scratch, \textsc{Luna} also supports \emph{post-hoc conversion} of quadratic models: given a fine-tuned checkpoint (e.g., BERT-base on GLUE or ViT-B/16 on ImageNet-1K), we replace each softmax attention layer with its \textsc{Luna} linear counterpart and briefly fine-tune to recover accuracy, thereby avoiding reliance on the exponential softmax kernel while remaining compatible with existing architectures. Theoretically, for a single-layer instantiation, we derive a feature-level Rademacher complexity bound showing that, under standard norm and Lipschitz assumptions, the hypothesis class induced by our learned kernel family has complexity scaling as $\tilde{\mathcal{O}}(1/\sqrt{n})$ with controlled dependence on that family.

\noindent\textbf{Contributions~}in this work are summarized as:
\begin{itemize}
\item We introduce a positive-definite, kernelized attention with \emph{learnable} feature maps that subsume fixed random-feature schemes (e.g., random Fourier features) while preserving linear-time/memory structure. This design enables the model to discover kernels auto-adaptive to each modality and task distribution.
\item We provide a concise PD guarantee for our construction and an approximation-error decomposition (parametrization vs.\ sampling), with high-probability bounds under bounded and unbounded feature regimes. This design provides theoretical justification for the method's generalization capabilities.
\item We show that softmax attention in \emph{finetuned} models can be replaced by \textsc{LUNA} with brief task-specific finetuning, recovering most of the original performance for BERT on GLUE and ViT-B/16 on ImageNet-1K, and outperforming exponential-feature linearizations. This post-hoc conversion capability enables practical deployment of linear attention in existing production systems without re-training from scratch.
\item Under matched compute, \textsc{LUNA} sets a new state-of-the-art average accuracy on the Long Range Arena.

\end{itemize}

\section{Related Work}
\begin{figure*}[!t]
  \centering
  \includegraphics[width=\textwidth]{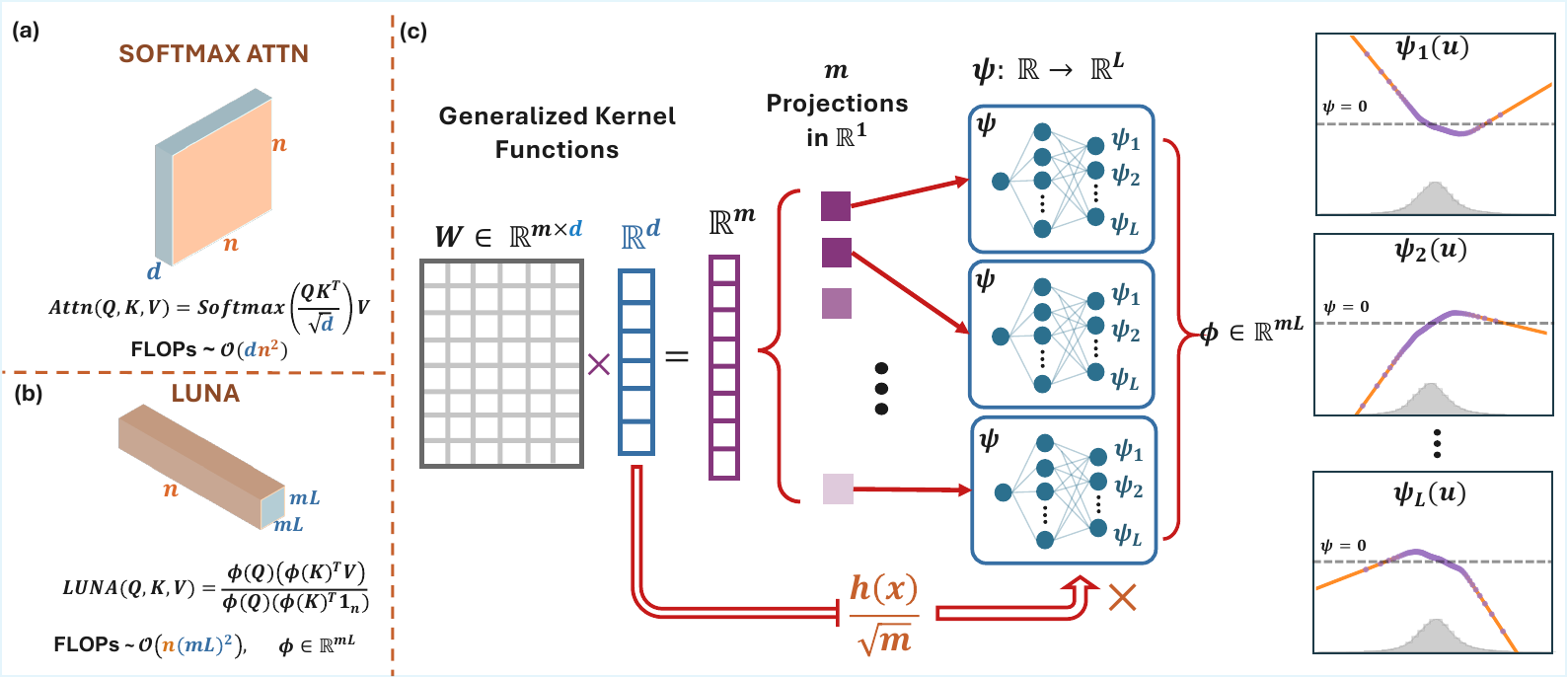}
  \caption{
    (a) Softmax attention requires computing all pairwise interactions among tokens, which causes the cost to grow quadratically with the sequence length. (b) LUNA introduces a learnable kernel method for linearizing the attention mechanism, shifting the expensive step from the sequence length $n$ to the feature map size $mL$. (c) For a given set of tokens, LUNA applies $m$ linear projections $W_i \in \mathbb{R}^{d}$, producing $m$ scalar values. Each scalar is then passed through a shared MLP $\psi \colon \mathbb{R} \rightarrow \mathbb{R}^{L}$. By concatenating the $L$ outputs across all $m$ projections, we obtain the kernel feature map $\phi \in \mathbb{R}^{mL}$. The plots on the right show several learned components $\psi_i$ for the LRA-Text task, illustrating that the resulting scalar functions differ from the commonly used fixed choices such as $\tanh$, $\sin$, or $\exp$. The gray band/histogram represents the empirical distribution of $u$.
    }

  \label{fig:teaser}
\end{figure*}
\subsection{Quadratic Attention and Stochastic Structure}
The original Transformer uses a quadratic-time softmax attention whose cost becomes prohibitive at long context lengths \citep{vaswani2017attention}. Within this quadratic regime, some studies exploit the stochastic structure of attention matrices. Standard softmax attention is row-stochastic but not column-normalized; doubly-stochastic variants enforce approximate bistochasticity via Sinkhorn-style normalization and related transport-inspired constraints, mainly to improve stability and interpretability while keeping $\mathcal{O}(n^2)$ complexity \citep{sander2022sinkformers, shahbazi2025espformer}. These methods act on the normalization of the attention \emph{matrix} and are largely orthogonal to linearization: such constraints can be combined with either quadratic softmax attention or linearized kernels \citep{shahbazi2025lotformerdoublystochasticlinearattention}.

Beyond such normalization-based approaches, Teo and Nguyen view self-attention through kernel PCA and propose a robust quadratic variant (RPC-Attention), whereas we work in the linear-attention setting and learn the kernel feature family itself \citep{teo2025unveiling}.

\subsection{From Quadratic to Efficient Approximations}
To remove the $\mathcal{O}(n^2)$ bottleneck, efficient attention mechanisms replace dense all-to-all interactions with sparse, low-rank, or kernelized surrogates. Structured-sparsity methods restrict each token to a subset of neighbors using local windows, dilations, or block patterns (e.g., local attention, Longformer, and hashing/block-sparse designs) \citep{parmar2018imagetransformer,beltagy2020longformer,Kitaev2020Reformer,NEURIPS2020_c8512d14}. Kernelized and low-rank approaches factorize similarity via feature maps or learned projections: Linear Transformers and Performer re-express softmax attention through (random or deterministic) kernel features, turning $n \times n$ accumulations into $n \times D$ computations \citep{katharopoulos2020transformers,choromanski2020rethinking}, while Linformer, Nystr\"omformer, and related methods learn or approximate a low-rank structure in the keys, values, or kernel matrix \citep{wang2020linformer,xiong2021nystromformer}. Subsequent work proposes orthogonal or structured bases and stabilized feature maps \citep{NEURIPS2021_c0f168ce,chen2021skyformer,zhen2022cosformer}, and Synthesizer departs from query-key matching by learning synthetic mixing weights \citep{tay2021synthesizer}. ``Post-hoc conversion'' techniques replace softmax with an exponential feature map in pretrained models and finetune to recover performance, prioritizing compatibility with existing checkpoints over expressivity of the feature map itself \citep{kasai2021t2r,zhang2024hedgehog}. These studies define the linear-attention paradigm: approximate softmax attention while reducing complexity to linear or near-linear in token length.

\subsection{Positioning \textsc{LUNA}}
\textsc{LUNA} operates on the \emph{feature-map} side of linear attention. Instead of fixing the kernel feature map \emph{a priori} (e.g., via random Fourier features or hand-crafted nonlinearities) or focusing primarily on compressing queries, keys, or values, \textsc{LUNA} learns a task-specific family of kernel features. Concretely, it combines learnable input projections with a shared bank of scalar channel functions, trained end-to-end on a single task. This design separates specialization over inputs from the representation power of the nonlinearity, allowing the feature basis to adapt to the data while preserving the drop-in efficiency of linear attention (linear complexity in the sequence length and feature dimension). In this sense, \textsc{LUNA} differs from fixed or purely random feature maps used in Performer/Cosformer/Skyformer-style designs and from exponential-feature conversions used for post-hoc linearization, providing a learned kernel representation tailored to the task within a standard linear-attention pipeline.

Importantly, this differs from post-hoc conversion approaches such as T2R and Hedgehog \citep{kasai2021t2r,zhang2024hedgehog}, which replace softmax with a fixed exponential feature map and fine-tune to \emph{mimic} the softmax kernel. \textsc{LUNA} does not attempt to match softmax; it learns a task-aligned kernel feature family directly from data. The exponential map is a special case within our parameterization, but it is not a constraint, enabling \textsc{LUNA} to retain linear complexity while searching a richer space of kernels better suited to the downstream task.

\section{Learning Kernels for Linear Attention}
We adopt the kernelized formulation of dot-product attention and its random-feature linearizations. This view makes explicit the sufficient-statistics factorization that underlies linear-time, linear-memory variants and recovers softmax attention as a specific positive-definite kernel induced by an exponential feature map. Building on this formulation, we introduce a parametric family of learnable feature maps that (i) define a valid positive-definite kernel and (ii) retain the linear compute pattern of kernelized attention layers.

\subsection{Preliminaries: Attention as a Kernel Method}

Let \(Q,K,V\in\mathbb{R}^{n\times d}\) denote the query, key, and value matrices, respectively, with sequence length \(n\) and latent dimension \(d\).
Scaled dot-product attention is given by
\begin{equation}
\eqfit{
\mathrm{Attn}(Q,K,V)
=\softmax \Big(\tfrac{QK^\top}{\sqrt d} \Big)V
=\frac{A V}{A \mathbf{1}_n},\quad
A=\exp\!\Big(\tfrac{QK^\top}{\sqrt d} \Big),
}
\end{equation}
where the softmax and the normalization by \(\mathbf{1}_n\in\mathbb{R}^{n}\) act row-wise. This admits a kernel view with
\(
k_{\mathrm{SM}}(x,y)=\exp\!\Big(\frac{x^\top y}{\sqrt{d}}\Big),
\qquad
x=Q[i,:],\;y=K[j,:].
\)
Although \(k_{\mathrm{SM}}\) is not shift-invariant, it is linked to the Gaussian kernel \(k_G(x,y)=\exp(-\tfrac{1}{2\sqrt{d}}\|x-y\|^2)\) via
\begin{equation}
k_{\mathrm{SM}}(x,y)=e^{\|x\|^2/2\sqrt{d}}\,k_G(x,y)\,e^{\|y\|^2/2\sqrt{d}}
\end{equation}
By Bochner’s theorem~\cite{rahimi2007random,rudin2017fourier}, a continuous, shift-invariant, positive-definite kernel admits
\begin{equation}\label{eq:shift_kernel}
\eqfitalign{
k(x-y)
&=\int_{\mathbb{R}^d} e^{i\omega^\top(x-y)}\,d\mu(\omega) \\
&=\mathbb{E}_{\omega\sim\mu}\!\big[e^{i\omega^\top(x-y)}\big] =\mathbb{E}_{\omega\sim\mu,\; b\sim\mathrm{Unif}(0,2\pi)} 
   \!\big[\zeta_{\omega,b}(x)\,\zeta_{\omega,b}(y)\big],
}
\end{equation}
with \(\zeta_{\omega,b}(x)=\sqrt{2}\cos(\omega^\top x+b)\), where \(b\sim\mathrm{Unif}(0,2\pi)\) is added as a variance-reduction trick to get cosine-shifted features~\cite{rahimi2007random}. Writing \(k_\mu(x,y) := k(x-y)\) for the kernel induced by the spectral measure \(\mu\), we approximate the expectation in~\eqref{eq:shift_kernel} with \(m\) Monte Carlo samples \((\omega_i,b_i)\overset{\text{i.i.d.}}{\sim}\mu\times\mathrm{Unif}(0,2\pi)\) obtaining the finite feature map:
\begin{equation}
\phi_m(x)
=
\sqrt{\frac{2}{m}}
\big[\cos(\omega_1^\top x + b_1),\ldots,\cos(\omega_m^\top x + b_m)\big]^\top,
\end{equation}
yielding the empirical kernel estimator
\begin{equation}
\eqfitalign{
\hat{k}_\mu^{(m)}(x,y)
&:= \phi_m(x)^\top \phi_m(y) \\
&= \frac{2}{m}\sum_{i=1}^m \cos(\omega_i^\top x + b_i)\,\cos(\omega_i^\top y + b_i),
}
\end{equation}
which converges to \(k_\mu(x,y)\) as \(m\to\infty\). Specializing to the Gaussian spectral measure \(\mu=\mathcal{N}(0,\tfrac{1}{\sqrt{d}}I_d)\) yields the standard random Fourier features (RFF) approximation to the Gaussian kernel,
\begin{equation}
k_G(x,y)\approx \hat{k}_\mu^{(m)}(x,y)
= \phi_m(x)^\top \phi_m(y),
\end{equation}

\paragraph{Performer features.}
To linearize the exponential dot-product kernel
\(
k_{\exp}(x,y)=\exp(x^\top y/\sqrt d),
\)
Performer~\cite{choromanski2020rethinking} uses the fact that this kernel admits a Gaussian-moment factorization:
\begin{equation}
\eqfitalign{
k_{\exp}(x,y)
&= \exp\!\Big(\tfrac{x^\top y}{\sqrt d}\Big) \\
&= \mathbb{E}_{\omega\sim\mathcal{N}(0,I_d)}
\!\Big[
e^{\frac{2\omega^\top x - \|x\|^2}{2\sqrt d}}\,
e^{\frac{2\omega^\top y - \|y\|^2}{2\sqrt d}}
\Big].
}
\end{equation}
This representation expresses the kernel as an expectation of two separated functions of \(x\) and \(y\), enabling a random-feature approximation.  
Drawing \(\omega_1,\ldots,\omega_m\overset{\text{i.i.d.}}{\sim}\mathcal{N}(0,\frac{1}{\sqrt{d}}I_d)\) produces the Performer feature map
\begin{equation}
\eqfitalign{
\phi^{P}_m(x)
&=
\frac{1}{\sqrt m}
\big[
e^{\frac{-\|x\|^2+2\omega_1^\top x}{2\sqrt d}},\ldots,
e^{\frac{-\|x\|^2+2\omega_m^\top x}{2\sqrt d}}
\big]^\top,
}
\end{equation}
for which
\(
\phi^{P}_m(x)^\top \phi^{P}_m(y)\approx k_{\exp}(x,y).
\)

\paragraph{Linear attention.}
With a (possibly learned) feature map \(\phi:\mathbb{R}^d\to\mathbb{R}^D\) applied row-wise on the query and key matrices, we have
\begin{equation}
\label{eq:kernel-attn}
\eqfit{
\mathrm{Attn}(Q,K,V)
= \frac{\phi(Q)\big(\phi(K)^\top V\big)}
         {\phi(Q)\big(\phi(K)^\top \mathbf{1}_n\big)},
}
\end{equation}
reducing complexity from \(\mathcal{O}(n^2d)\) to \(\mathcal{O}(nD^2)\). The denominator is computed element-wise per row.
\subsection{Our Method: Fully Learnable Kernels}
\label{sec:learned-kernel}

The constructions above show that many efficient Transformers admit a kernel of the form
\(
k(x,y)
= \mathbb{E}_{\omega\sim\mu}\big[\zeta_\omega(x)\,\zeta_\omega(y)\big],
\)
where $\zeta_\omega$ is a \emph{fixed}, scalar feature determined by the choice of spectral measure $\mu$ (Fourier features, Performer, etc.).  
In all such cases the kernel class is hard-coded by $\zeta_\omega$. We generalize this template by replacing the scalar feature $\zeta_\omega(x)$ with a vector-valued, \emph{learnable} feature family $\phi_\omega(x)$.  
Formally, we define
\begin{equation}
\label{eq:kernel}
\eqfitalign{
k(x,y)
&\coloneqq \mathbb{E}_{\omega\sim\mathcal{N}(0,I_d)}
\big\langle \phi_\omega(x),\phi_\omega(y)\big\rangle_{\mathcal{H}} \\
&\approx \frac{1}{m}\sum_{i=1}^m
\big\langle \phi(x;\omega_i),\phi(y;\omega_i)\big\rangle_{\mathcal{H}},
}
\end{equation}
where each $\phi_\omega:\mathbb{R}^d\to\mathcal{H}$ maps an input to a feature in a Hilbert space
$\mathcal{H}\cong\mathbb{R}^{mL}$ endowed with the Euclidean inner product.\footnote{For convenience, we do not distinguish $\phi_\omega(\cdot)$, $\phi(\omega,\cdot)$, and $\phi(\cdot;\omega)$ in this paper.}
This preserves the positive-definite kernel structure and the linear-time, streaming form of kernel attention, while allowing the feature family and the kernel itself to adapt to the data.

\begin{proposition}\label{prop:kernel}
The construction in \eqref{eq:kernel} yields a positive-definite kernel. Conversely, by Mercer’s theorem, any positive-definite kernel admits such a representation for an appropriate \(\phi(\cdot;\omega)\) into an RKHS \(\mathcal{H}\). See Appendix \ref{sec:kernel_intro}.
\end{proposition}

We replace fixed RFF components with learnable projections and channels.  
Let $W\in\mathbb{R}^{m\times d}$ with rows $\{w_i^\top\}_{i=1}^m$, channel functions
$\{\psi_\ell:\mathbb{R}\!\to\!\mathbb{R}\}_{\ell=1}^L$ (each instantiated as a small MLP on the scalar projection $u = w_i^\top x$), and a tokenwise envelope $h:\mathbb{R}^d\to\mathbb{R}$.  
Our feature map is
\begin{equation}
\label{eq:phi-def-scalarenv}
\eqfit{
\phi(x;W,\psi,h)
= \frac{h(x)}{\sqrt m}\,
\big[\psi_\ell(w_i^\top x)\big]_{\substack{i=1,\dots,m\\ \ell=1,\dots,L}}
\in\mathbb{R}^{mL},
}
\end{equation}
 
This template strictly generalizes RFF and Performer features: both are recovered by fixing $h$ and the $\psi_\ell$ by hand, whereas in our case $W$, $h$, and the channel functions $\{\psi_\ell\}$ are learned from data.





\begin{remark}[Neural Approximation of Multiplicatively Decomposable Kernels]
Let $k:\mathbb{R}^d\times\mathbb{R}^d\to\mathbb{R}$ be a positive-definite kernel, and suppose it admits the multiplicative form
\[
k(x,y)=h(x)\,k'(x-y)\,h(y),
\]
for some continuous functions $h:\mathbb{R}^d\to\mathbb{R}$ and $k':\mathbb{R}^d\to\mathbb{R}$.  
Under suitable regularity assumptions on $h$, $k'$, and the domain $K\subset\mathbb{R}^d$, there exists a neural feature map
$
\phi(x;\omega) \in \mathbb{R}^L,
$
of the form described in~\eqref{eq:phi-def-scalarenv}, such that with high probability over random draws $W=\{\omega_i\}_{i=1}^m$,
\[
\big|\,k(x,y)-k_{NN}(x,y;W)\,\big| \le \varepsilon,
\qquad \forall (x,y)\in K\times K,
\]
where the neural kernel estimator is
\[
k_{NN}(x,y;W)
:=\frac{1}{m}\sum_{i=1}^m \phi(x;\omega_i)^\top \phi(y;\omega_i).
\]
Here each feature map takes the separable form
\[
\phi(x;\omega_i)
= h(x)\,\big[\psi_\ell(\omega_i^\top x)\big]_{\ell=1}^L,
\]
which is one component of the structured feature construction in~\eqref{eq:phi-def-scalarenv}.
\end{remark}

The total approximation error naturally decomposes into two parts:

\begin{itemize}
    \item \textbf{Parametrization error}: the error arising from approximating the ideal feature map $\phi$ using a finite-depth neural network;
    \item \textbf{Sampling (or generative) error}: the error incurred by replacing the population expectation with a finite set of random parameters $\{\omega_i\}_{i=1}^m$.
\end{itemize}

The informal conclusion above summarizes the combined effect of Propositions~\ref{pro:error1_main} and \ref{pro:error2_main}, stated next.


\begin{proposition}[\textbf{Parametrization Error (Informal)}]
\label{pro:error1_main}
Suppose the kernel admits the decomposition $k(x,y)=h(x)\,k'(x-y)\,h(y)$, with $h$ and $k'$ continuous.  
Under suitable regularity assumptions, there exist MLP-based feature maps $\phi(x;\omega)$ of the form~\eqref{eq:phi-def-scalarenv} such that, 
\[
\big|\,k(x,y)-k_{NN}(x,y)\,\big| < \varepsilon,
\qquad \forall x,y\in K,
\]
where
\[
k_{NN}(x,y)
:=\mathbb{E}_{\omega\sim\mathcal{N}(0,I_d)}
\big[\phi(x;\omega)^\top \phi(y;\omega)\big],
\]
and $K\subset\mathbb{R}^d$ is a fixed compact domain. 
\end{proposition}

This proposition generalizes the classical \emph{universal approximation theorem} from function approximation to \emph{kernel approximation via neural feature maps} in a linear-attention–style construction. The precise statements are deferred to Appendix~\ref{sec:param_error}, where we treat separately:

\begin{itemize}
    \item bounded kernels with bounded activations (Proposition~\ref{pro:error_1_bounded});
    \item unbounded kernels or unbounded activations (Proposition~\ref{pro:error1_unbounded}).
\end{itemize}


\begin{proof}[Sketch of the proof of Proposition \ref{pro:error1_main}]
Under the assumption $k(x,y)=h(x)\,k(x-y)\,h(y)$, we have
\begin{align}
&k(x,y)=\mathbb{E}_{\omega\sim\mathcal{N}(0,I_d)}
\big[\phi(\omega,x)^\top \phi(\omega,y)\big],\nonumber \\
&\phi(\omega,x)=h(x)\psi(x).\nonumber 
\end{align}
By the Universal Approximation Theorem, there exist one-hidden-layer neural networks 
$h_{\mathrm{NN}}$ and $\psi_{\mathrm{NN}}$ such that $h_{\mathrm{NN}}\approx h$ and 
$\psi_{\mathrm{NN}}\approx\psi$ uniformly on $K$. Hence 
\[
\phi_{\mathrm{NN}}(\omega,x):=h_{\mathrm{NN}}(x)\psi_{\mathrm{NN}}(\omega,x)
\]
satisfies $\phi_{\mathrm{NN}}(\omega,x)\approx \phi(\omega,x)$, and therefore the induced kernel
\[
k_{\mathrm{NN}}(x,y)
:=\mathbb{E}_{\omega\sim\mathcal{N}(0,I_d)}
\big[\phi_{\mathrm{NN}}(\omega,x)^\top \phi_{\mathrm{NN}}(\omega,y)\big]
\]
approximates the target kernel $k(x,y)$ uniformly on $K\times K$.
\end{proof}




\begin{proposition}[\textbf{Sampling Error (Informal)}]
\label{pro:error2_main}
Assume the kernel admits the representation
\[
k(x,y)
=\mathbb{E}_{\omega\sim\mathcal{N}(0,I_d)}
\big[\phi(x;\omega)^\top \phi(y;\omega)\big].
\]
Let $\omega_1,\ldots,\omega_m \overset{\text{i.i.d.}}{\sim}\mathcal{N}(0,I_d)$, and define the Monte Carlo estimator
\[
k(x,y;W)
= \frac{1}{m}\sum_{i=1}^m 
\phi(x;\omega_i)^\top \phi(y;\omega_i).
\]
Then, under suitable regularity conditions, for any $x,y\in K$,
\[
\mathbb{P}\!\left(\big|k(x,y)-k(x,y;W)\big|\ge \varepsilon\right)
\le \delta,
\]
where $\delta$ decays exponentially in $m$.
\end{proposition}

\noindent The formal results are provided in Appendix~\ref{sec:error_omega}, where we characterize the sampling error in two regimes:

\begin{itemize}
    \item \textbf{Bounded feature regime.}  
    If $\phi(x;\omega)$ is uniformly bounded for all $x \in K$, then by Gaussian concentration,
    $$
    \delta = \mathcal{O}(\exp(-c m \epsilon^2)),
    $$
    for some constant $c>0$. See Corollary \ref{cor:err2_bound} and Remark~\ref{rm:error2_bound} and related results for details.

    \item \textbf{Unbounded feature regime.}  
    If $\{\phi(x;\omega):x\in K\}$ is not uniformly bounded (e.g., ReLU activations), then an exponential-type concentration bound gives
    $$
    \delta = \mathcal{O}(\exp(-\min(c_1 \epsilon^2 m, c_2 \epsilon m))),
    $$
    where $c_1, c_2 > 0$ are constants. See Corollary~\ref{cor:error2_unbounded} for the precise statement.
\end{itemize}

\begin{proof}[Sketch of proof of Proposition \ref{pro:error2_main}]
Fix arbitrary $x,y\in K$. Since $\omega\sim\mathcal{N}(0,I_d)$, both $\phi(\omega,x)$ and $\phi(\omega,y)$ are sub-Gaussian. Consequently,
\[
A(\omega):=\phi(\omega,x)^\top \phi(\omega,y)
\]
is sub-Gaussian when $\{\phi(\cdot,x):x\in K\}$ is uniformly bounded, and sub-exponential otherwise.  
The probability term in Proposition \ref{pro:error2_main} can therefore be written as
\[
\mathbb{P}\!\left(\left|\frac{1}{m}\sum_{i=1}^m A(\omega_i)-\mathbb{E}[A(\omega)]\right|>\epsilon\right),
\]
which is directly controlled by standard concentration inequalities for sums of i.i.d. sub-Gaussian or sub-exponential random variables.
\end{proof}

\section{Runtime Analysis}
\label{sec:runtime}

Let $n$ be the sequence length, $d$ the key/query width, $d_v$ the value width, and $D=mL$ the feature dimension of $\phi(x;W,\psi,h)\in\mathbb{R}^{D}$. Using the kernelized form of Eq.~\eqref{eq:kernel-attn}, the per-head cost decomposes into (i) evaluating $\phi(\cdot)$ for all tokens and (ii) forming two key-side sufficient statistics followed by a single query-side application:
\[
S_{KV}=\phi(K)^\top V\in\mathbb{R}^{D\times d_v},\qquad
S_{K1}=\phi(K)^\top \mathbf{1}_n\in\mathbb{R}^{D}.
\]
The resulting complexity is
\[
T_{\text{total}}
=\mathcal{O}\!\big(n\,(m d + D\,c_\psi)\big)\;+\;\mathcal{O}\!\big(n D d_v\big),
\]
where $c_\psi$ denotes the per-channel MLP cost in $\psi$. In the common regime $d_v=\Theta(D)$, this simplifies to
\begin{equation}
T_{\text{total}}=\mathcal{O}\!\big(n D^2\big)
=\mathcal{O}\!\big(n\,(mL)^2\big).
\end{equation}
The computation is \emph{linear in $n$} because the $n\times n$ attention matrix is never formed. \textit{Figure~\ref{fig:runtime_linear}} compares the resulting compute flow and empirical scaling against representative linear-attention baselines. Although all linear variants share $\mathcal{O}(nD^2)$ complexity, the additional cost of our learned feature maps is modest in practice: measured runtimes remain comparable to existing linear baselines, while the learned kernel yields performance improvements, as we discuss next.

\begin{figure}[t]
  \centering
  \includegraphics[width=0.5\columnwidth]{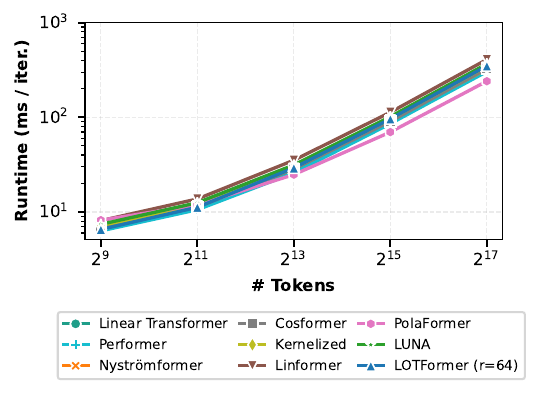}
  \caption{\textbf{Scaling of linear attention variants.} Per-layer runtime as a function of sequence length $n$ (log–log scale) for representative linear-attention baselines and \textsc{LUNA}, measured under identical settings. All methods exhibit approximately linear growth in $n$, with \textsc{LUNA} matching the runtime envelope of existing baselines while using learnable kernel features.}
  \label{fig:runtime_linear}
\end{figure}


\section{Experiments}
\label{sec:experiments}

We evaluate \textsc{LUNA} in two complementary settings: (i) the heterogeneous, long-context \emph{Long Range Arena (LRA)} benchmark, and (ii) \emph{post-hoc} conversion of finetuned quadratic-attention models in language (BERT-base on GLUE) and vision (ViT-B/16 on ImageNet-1K). Across all experiments, we keep the backbone architecture, optimizer, schedule, and training budget fixed within each setting, and vary only the attention module and its feature dimension $D = mL$ so that competing linear-attention variants are compared at matched compute.

\begin{figure*}[ht]
  \centering
  \includegraphics[width=\textwidth]{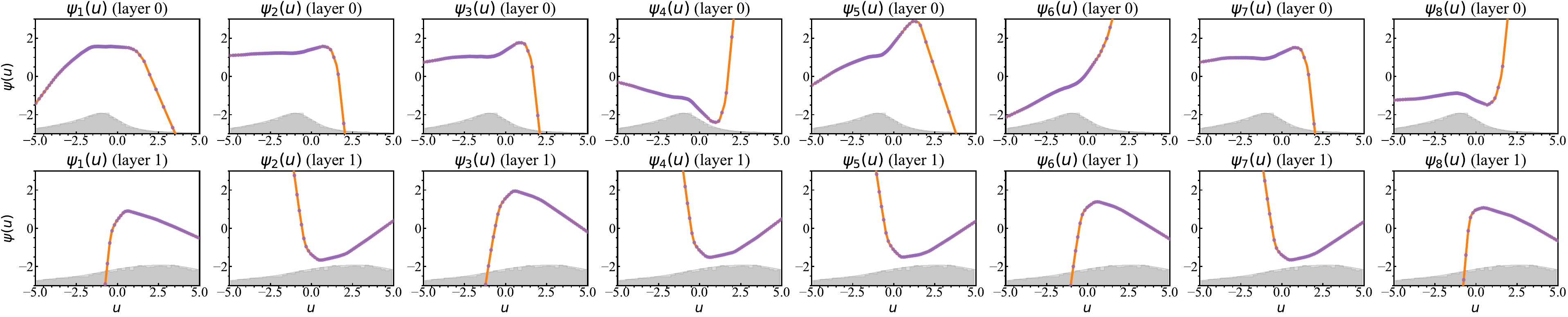}
\caption{Channel-wise visualization of the learned feature map $\phi$ on LRA--Image with $M{=}8$ and $L{=}8$. Each small subplot shows a channel MLP $\psi_\ell(u)$  versus the scalar projection $u{=}w_i^\top x{+}b_i$. Top row: Transformer-layer 0; Bottom row:  Transformer-layer 1; Orange curve: $\psi_\ell(u)$ evaluated on real $u$; Gray histogram: the empirical distribution of $u$.}
\label{fig:lra_channels}
\end{figure*}

\subsection{Long Range Arena}
\label{sec:lra}
The \emph{Long Range Arena (LRA)} evaluates long-context reasoning across heterogeneous modalities (Text, ListOps, Retrieval, Pathfinder, Image) under a shared training protocol. Within this setting, we compare \textsc{Luna} to softmax attention and representative linear-attention baselines at matched compute, keeping the backbone, optimizer, schedule, and training budget fixed and varying only the attention module and its feature dimension \(D = mL\) (number of projections \(m\) and channels \(L\)).

\paragraph{Results.}
Table~\ref{tab:lra_results} summarizes test accuracy by task and averaged across LRA. \textsc{Luna} matches or surpasses prior efficient Transformers on four of five tasks, with the largest gains on Image and strong performance on Text, while operating in the same linear-attention regime as the baselines.

\begin{table}[t]
\caption{Results on the LRA benchmark. Best numbers per column are in bold. Our method achieves the highest overall average accuracy (\%), with particularly strong gains on Text, ListOps, and Image.}
\label{tab:lra_results}
\centering
\scriptsize
\resizebox{0.75\textwidth}{!}{%
\begin{tabular}{lccccc c}
\toprule
Model & Text & ListOps & Retrieval & Pathfinder & Image & Avg. \\
\midrule
Transformer      & 61.55 & 38.71 & 80.93 & 70.39 & 39.14 & 58.14 \\
LocalAttn        & 52.98 & 15.82 & 53.39 & 66.63 & 41.46 & 46.06 \\
LinearTrans.     & 65.90 & 16.13 & 53.09 & 75.30 & 42.34 & 50.55 \\
Reformer         & 56.10 & 37.27 & 53.40 & 68.50 & 38.07 & 50.67 \\
Performer        & 65.40 & 18.01 & 53.82 & \textbf{77.05} & 42.77 & 51.41 \\
Synthesizer      & 61.68 & 36.99 & 54.67 & 69.45 & 41.61 & 52.88 \\
Longformer       & 62.85 & 35.63 & 56.89 & 69.71 & 42.22 & 53.46 \\
Informer         & 62.13 & 37.05 & 79.35 & 56.44 & 37.86 & 54.57 \\
Bigbird          & 64.02 & 36.05 & 59.29 & 74.87 & 40.83 & 55.01 \\
Linformer        & 57.29 & 36.44 & 77.85 & 65.39 & 38.43 & 55.08 \\
Kernelized       & 60.02 & 38.46 & 82.11 & 69.86 & 32.63 & 56.62 \\
Cosformer        & 63.54 & 37.20 & 80.28 & 70.00 & 35.84 & 57.37 \\
Nystrom          & 62.36 & 37.95 & 80.89 & 69.34 & 38.94 & 57.90 \\
Skyformer        & 64.70 & 38.69 & 82.06 & 70.73 & 40.77 & 59.39 \\
Hedgehog         & 64.60 & 37.15 & \textbf{82.24} & 74.16 & 40.15 & 59.66 \\
PolaFormer$_{\alpha=3}$ & 73.06 & 37.35 & 80.50 & 70.53 & 42.15 & 60.72 \\
LOTFormer                  & 71.1 & 38.5 & 80.9 & 69.9 & 54.1 & 62.9 \\
\midrule
\textsc{LUNA}             & \textbf{73.41} & \textbf{38.94} & 81.02 & 69.52 & \textbf{64.32} & \textbf{65.44}  \\
\bottomrule
\end{tabular}
}
\end{table}

\noindent To probe how \textsc{LUNA} achieves its gains, we visualize representative channels of the learned kernel feature map $\phi(\cdot)$ on a held-out LRA--Image example (Figure~\ref{fig:lra_channels}).

\subsection{Post-hoc Conversion of Finetuned Quadratic Transformers}

We study \emph{post-hoc conversion} of finetuned quadratic-attention models to linear attention, followed by brief task-specific finetuning. Concretely, we start from BERT-base models finetuned on GLUE (BERT-FT), replace each softmax attention layer with a kernelized linear attention module, and finetune on the original task using the same optimizer, schedule, and training budget as in the baselines. Within this setting, we compare \textsc{LUNA} to exponential-feature conversions (T2R, T2R-HH) and Hedgehog under identical training recipes and feature dimensions. Table~\ref{tab:glue_conversion} reports dev-set scores and the percentage of the original BERT-FT performance recovered by each converted model.

Across GLUE, \textsc{LUNA} recovers 99.5\% of the BERT-FT score on average, slightly exceeding Hedgehog (99.3\%) and substantially improving over T2R (88.9\%) and T2R-HH (93.5\%). Task-wise, \textsc{LUNA} closely tracks BERT-FT on all benchmarks and modestly improves on SST2 and STS-B. These results indicate that a learned kernel feature map can act as a drop-in replacement for quadratic attention in finetuned BERT models, without any bespoke distillation or changes to the training protocol.

\begin{table*}[h]
\caption{Post-hoc conversion of \textbf{BERT-base} on GLUE. Scores are dev-set metrics in the standard GLUE format. "Recover" is the percentage of the BERT-FT score recovered by each converted model, averaged across tasks.}
\label{tab:glue_conversion}
\centering
\footnotesize
\resizebox{0.9\textwidth}{!}{%
\begin{tabular}{lcccccccccc}
\toprule
Method & CoLA & SST2 & MRPC & STS-B & QQP & MNLI & QNLI & RTE & (\%) Recover \\
\midrule
BERT-FT   & 58.8 & 93.2 & 90.2 & 88.8 & 91.0 & 84.7 & 91.3 & 68.2 & 100.0 \\
\midrule
T2R       & 43.6 & 87.7 & 83.0 & 78.6 & 86.7 & 78.9 & 84.6 & 54.1 & 88.9 \\
T2R-HH    & 56.9 & 90.9 & 89.1 & 77.7 & 90.0 & 77.4 & 84.5 & 56.3 & 93.5 \\
Hedgehog  & 59.2 & 92.6 & 90.1 & 87.4 & 91.0 & 82.6 & 89.6 & 69.3 & 99.3 \\
\midrule
\textsc{LUNA}      & 58.8 & 93.4 & 90.1 & 88.5 & 90.7 & 83.5 & 90.6 & 68.8 & \textbf{99.5} \\
\bottomrule
\end{tabular}
}
\end{table*}

\begin{figure*}[ht!]
  \centering
  \includegraphics[width=\textwidth]{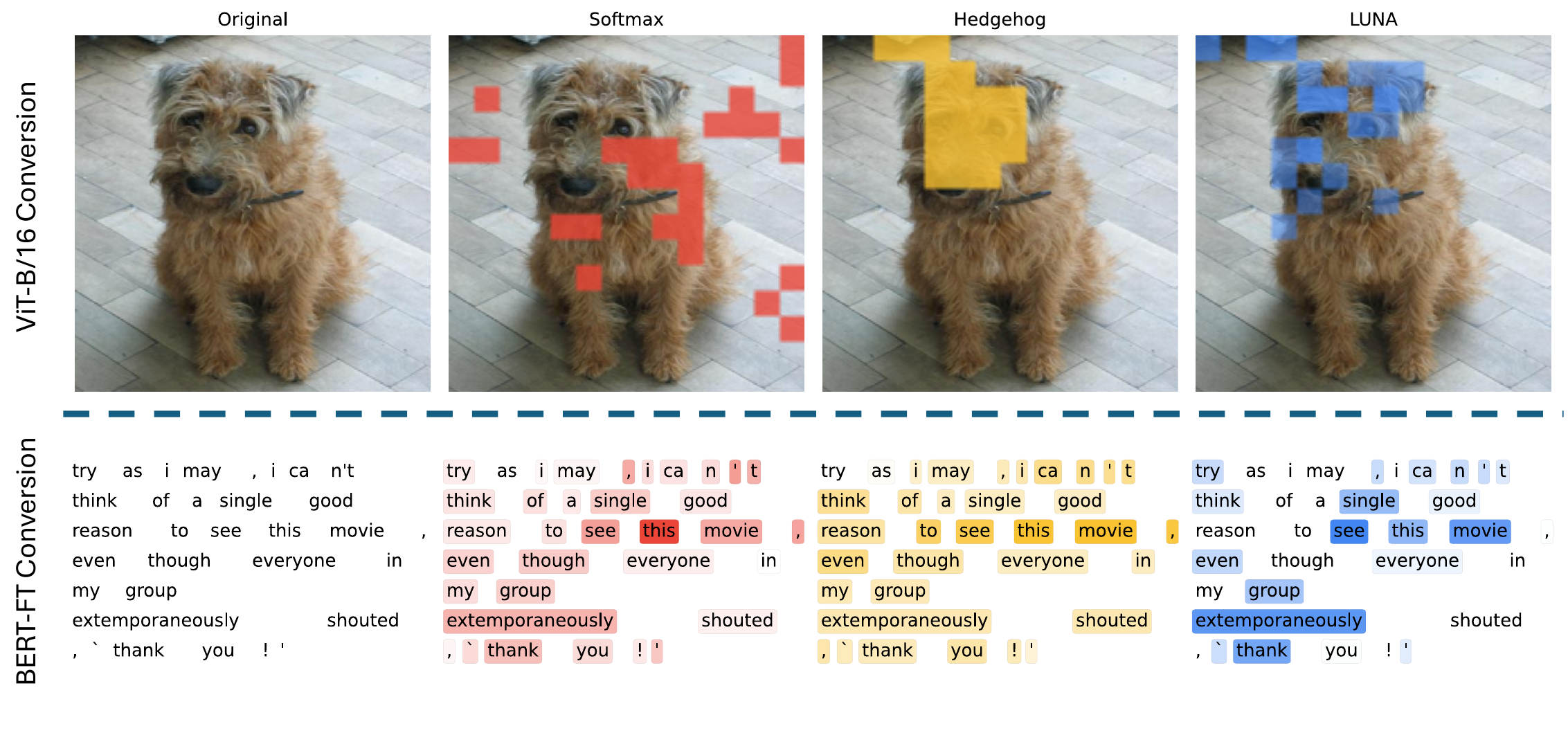}
  \vspace{-1mm}
  \caption{CLS-based attention visualizations across vision and language. Top row (ViT-B/16, ImageNet dog sample): CLS attention on images. Columns: Original, Softmax rollout, Hedgehog (diffuse baseline), \textsc{LUNA} (ours). \textsc{LUNA} concentrates on semantically coherent regions (eyes/snout/collar) and suppresses background, while Softmax is scattered and Hedgehog is overly smooth. Bottom row (BERT on SST-2): \textbf{CLS$\rightarrow$token attention on SST-2.} Columns: Original text, Softmax, Hedgehog, \textsc{LUNA}. Tokens are shaded with an alpha proportional to the normalized attention weight from the CLS token (red for Softmax, yellow for Hedgehog, blue for \textsc{LUNA}). \textsc{LUNA} yields compact, sentiment-aligned highlights comparable to Softmax and sharper than Hedgehog.}
  \label{fig:cls-viz-sst2}
\end{figure*}


\noindent We repeat the conversion experiment in vision using ViT-B/16 trained on ImageNet-1K. Starting from a finetuned softmax model, we swap softmax attention for \textsc{LUNA} (or a baseline linearization) in all layers and briefly finetune with the original training configuration. As shown in Table~\ref{tab:vit_imagenet}, \textsc{LUNA} attains 80.5\% Top-1 accuracy, slightly \emph{above} the original softmax model (80.3\%) and higher than T2R-HH (77.0\%) and Hedgehog (79.5\%). Thus, in this setting as well, the learned feature map preserves or modestly improves performance after conversion.

To probe qualitative behavior across modalities, we inspect CLS-centered attention maps on SST-2 and ImageNet-1K (Figure~\ref{fig:cls-viz-sst2}). On SST-2, \textsc{LUNA} produces sharper, content-aligned highlights than Hedgehog while remaining close to the softmax baseline: attention mass is concentrated on sentiment-bearing spans, with stopwords and punctuation receiving lower weight, whereas Hedgehog tends to spread mass more diffusely across the sequence. On ImageNet-1K validation images, we render CLS$\rightarrow$patch attention weights as alpha-tinted overlays; \textsc{LUNA} again yields compact, object-centric highlights with fewer off-object activations than Hedgehog, while staying qualitatively close to the softmax baseline.


\begin{table}[h]
\caption{Post-hoc conversion of \textbf{ViT-B/16} on ImageNet-1K validation accuracy (\%).}
\label{tab:vit_imagenet}
\centering
\small
\setlength{\tabcolsep}{4.5pt}
\renewcommand{\arraystretch}{1.2}
\begin{tabular}{lcccc}
\toprule
Top\text{-}1 & ViT-B/16 & T2R-HH & Hedgehog & \textsc{LUNA} \\
\midrule
Accuracy (\%) & 80.3 & 77.0 & 79.5 & \textbf{80.5} \\
\bottomrule
\end{tabular}
\end{table}

\subsection{Ablations on Channels and Projections}
\label{sec:lra_ablation}
We isolate the LRA--Image task to study how the number of \emph{projections} $M$ (feature projections per head) and the number of \emph{channels} $L$ (parallel attention channels) affect performance. Unless stated otherwise, the backbone, optimizer, schedule, and training recipe follow Section~\ref{sec:lra}. 

\noindent Table~\ref{tab:lra_ablate_LM} summarizes two 1D sweeps. At fixed $M{=}8$, accuracy is relatively stable as $L$ increases, with a mild peak at $L{=}8$ before degrading at $L{=}16$. Conversely, at fixed $L{=}8$, increasing $M$ beyond $8$ yields no benefit; $M{=}16$ is comparable to $M{=}8$, and performance drops sharply at $M{=}32$. This motivates our default choice of $L{=}8$ and $M{=}8$.


\begin{table}[t]
\centering
\caption{LRA--Image ablations. Top: vary channels $L$ (fix $M{=}8$). Bottom: vary projections $M$ (fix $L{=}8$). Report accuracy (\%).}
\label{tab:lra_ablate_LM}
\setlength{\tabcolsep}{5pt}
\renewcommand{\arraystretch}{1.08}
\small
\begin{tabular}{@{}lcccc@{}}
\toprule
\multicolumn{5}{c}{\textbf{Vary $L$ (channels), fix $M{=}8$}} \\
\midrule
$L$             & 2     & 4     & 8     & 16    \\
Accuracy (\%)   & 62.62 & 62.02 & 64.32 & 60.89 \\
\midrule
\multicolumn{5}{c}{\textbf{Vary $M$ (projections), fix $L{=}8$}} \\
\midrule
$M$             & 4     & 8     & 16    & 32    \\
Accuracy (\%)   & 63.06 & 64.32 & 63.88 & 59.15 \\
\bottomrule
\end{tabular}
\vspace{.2in}
\end{table}




\subsection{Ablations: Neural Kernel Variants}
\label{sec:lra_kernel_ablate}
We next ablate the neural kernel $\phi(\cdot)$ while keeping the backbone, training recipe, $(M,L)$, and DWC fixed. All variants share the same feature dimension $D$ and linear-attention implementation; only the form of $\phi(x)$ changes. We compare:
(i) \textbf{RFF--sin,cos (fixed)}: non-learnable sinusoidal/random features $\phi_{\sin,\cos}(x)$;
(ii) \textbf{Bank+Coef}: a basis $\{\phi_\ell(x)\}_{\ell=1}^L$ with learned \emph{global} coefficients $w_\ell$ (input-agnostic reweighting);
(iii) \textbf{Envelope-gated} $\mathrm{envelope}(x)\!\cdot\! \psi(x)$: a learned map $\psi(x)$ modulated by a tokenwise envelope $\mathrm{envelope}(x)\!\in\![0,1]^L$;
(iv) \textbf{No gate} $\psi(x)$: the same learned map with $\mathrm{envelope}(x)\!\equiv\!\mathbf{1}$. On LRA--Image, learned kernels are necessary: all neural variants outperform fixed RFF features, and the simple ungated map $\psi(x)$ substantially outperforms its envelope-gated counterpart. We therefore use $\phi(x)=\psi(x)$ as the default kernel design in \textsc{LUNA}.

\begin{table}[t]
\centering
\small
\caption{LRA--Image: neural kernel ablations (Top-1 accuracy).}
\label{tab:lra_neural_kernel_ablate}
\setlength{\tabcolsep}{1.8pt}
\renewcommand{\arraystretch}{1.08}
\begin{tabular}{l l c}
\toprule
Variant & Kernel form & Accuracy (\%) \\
\midrule
RFF--sin,cos (fixed) & $\phi(x)=\phi_{\sin,\cos}(x)$ & 35.72 \\
Bank+Coef            & $\phi(x)=\sum_{\ell=1}^{L} w_{\ell}\,\phi_{\ell}(x)$ & 40.30 \\
Envelope-gated       & $\phi(x)=\mathrm{h}(x)\odot \psi(x)$ & 41.52 \\
No gate $\psi(x)$       & $\phi(x)=\psi(x)$ & \textbf{64.32} \\
\bottomrule
\end{tabular}
\vspace{-0.25em}
\end{table}

\section{Conclusion}
We introduced \textsc{Luna}, a linear attention mechanism that replaces fixed, data-agnostic feature maps with a learned kernel feature map while preserving a positive-definite kernel and an associative, streaming formulation. This yields linear time and memory complexity and substantially narrows the accuracy gap to quadratic softmax attention. Under matched compute, \textsc{Luna} achieves state-of-the-art performance among efficient Transformers on Long Range Arena, and in post-hoc conversion experiments, swapping softmax in fine-tuned BERT and ViT-B/16 and briefly fine-tuning, it recovers most of the original performance, clearly outperforming fixed exponential-feature linearizations. Ablations show that jointly learning projections and channel functions is crucial, and that removing multiplicative gating improves optimization stability without degrading accuracy. Qualitative CLS-centric visualizations further indicate that \textsc{Luna} concentrates attention on semantically salient spans and regions in both text and images.

\section*{Acknowledgment}
SK acknowledges support from the NSF CAREER Award No.\ 2339898 and the Wellcome Leap “Surgery: Assess/Validate/Expand (SAVE)” program. AS acknowledges support from Lambda Labs through a Lambda Cloud Research Credit award.

{
    \small
    
    \bibliographystyle{cvpr2026-main/ieeenat_fullname}
    \bibliography{arxiv/template}
}

\clearpage
\onecolumn
\setcounter{page}{1}

\section{Supplementary Material}
\begin{remark}
For convenience, in the appendix, we do not distinguish $\phi(\omega,x),\phi_\omega(x)$ and $\phi(x;\omega)$ where $\phi$ is  the feature mapping used to define the kernel. 
\end{remark}
\section{Proof of Proposition \ref{prop:kernel}}\label{sec:kernel_intro}
For convenience, let $G=\mathcal{N}(0,I_d)$. We first introduce the following fundamental result for Kernel function and the feature space: 

\begin{lemma}\label{lem:kernel}
Let $k(x,y)$ be a positive definite kernel, then there exists a feature mapping $\phi:\mathbb{R}^d \to \mathcal{H}$ such that $k(x,y)=\langle \phi(x),\phi(y) \rangle_{\mathcal{H}}$ and $\text{dim}(\mathcal{H})\leq |\mathbb{R}|$. 
\end{lemma}
\begin{remark}
Given a kernel function, its feature mapping and the corresponded feature space are not uniquely determined. Thus, we only claim existence in the above statement. 
\end{remark}
\begin{proof}
We define the RKHS (Reproducing Kernel Hilbert Space): 
$$\mathcal{H}:=\overline{span}\{k(\cdot,x):x\in X\}$$
where the inner product is defined as the following: 
for each $f,g\in \mathcal{H}$, by definition of the RKHS space, there exists sequence $(a_i)_{i=1}^\infty,(b_i)_{i=1}^\infty\subset \mathbb{R}$ such that  $f=\lim_n f_n,g=\lim_n g_n$ where $f_n=\sum_{i=1}^n a_i k(x,x_i),g_n=\sum_{i=1}^n b_i k(x,x_i)$.
Then $$\langle f,g \rangle_{\mathcal{H}}:=\lim_{n}\langle f_n,g_n\rangle_{\mathcal{H}}:= \lim_{n} \sum_{i,j=1}^na_ib_jk(x_i,x_j).$$

\noindent It is clear $\phi(x):x\mapsto k(\cdot,x)$ is a feature mapping with $k(x,y)=\langle \phi(x),\phi(y) \rangle_{\mathcal{H}}$. It remains to show $\text{dim}(\mathcal{H}_\phi)\leq |\mathbb{R}|$. Choose an orthogonal basis of $\mathcal{H}$, denoted as $\{\phi_\alpha\}$, we have: 
$$|\{\phi_\alpha\}|\leq |\{(k(\cdot,x):x\in X\}|\leq |X|=|\mathbb{R}^d|=|\mathbb{R}|.$$
where the last line holds since  $d\in \mathbb{N}$ is finite. 
Thus, we complete the proof. 
\end{proof}

\paragraph{Forward direction.}  
It is immediate that $k(x,y)$ is symmetric. It remains to show positive definiteness.  
For any $c_1,\ldots,c_n\in \mathbb{R}$ and $x_1,\ldots,x_n\in \mathbb{R}^d$, we have
\begin{align}
\sum_{i,j=1}^n c_i c_j k(x_i,x_j) 
&= \mathbb{E}_{\omega\sim G}\!\left[\sum_{i,j=1}^n c_i c_j \phi_\omega(x_i)\cdot \phi_\omega(x_j)\right] \nonumber\\
&= \mathbb{E}_{\omega\sim G}\!\left[\Big\|\sum_{i=1}^n c_i \phi_\omega(x_i)\Big\|^2\right] \nonumber\\
&\ge 0. \nonumber
\end{align}
Hence, $k(x,y)$ defined in \eqref{eq:kernel} is a positive definite kernel.  

\paragraph{Reverse direction.}  
Suppose $k(x,y)$ is a positive definite kernel. Then by the above lemma, there exists a feature map $\phi:\mathbb{R}^d \to \mathcal{H}$ into a Hilbert space $(\mathcal{H},\langle\cdot,\cdot\rangle_\mathcal{H})$ such that
$$
k(x,y) = \langle \phi(x),\phi(y)\rangle_\mathcal{H}.
$$
and there exists an orthogonal basis $U=\{u_\alpha\}_{\alpha\in \mathbb{R}}$ spanning $\mathcal{H}$. Define:
$$
g_\alpha(x) := \langle \phi(x),u_\alpha \rangle_\mathcal{H}.
$$
Then
\begin{align}
k(x,y) = \int_{\alpha\in \mathbb{R}} g_\alpha(x) g_\alpha(y) d\alpha. \label{pf:k}
\end{align}

\noindent We distinguish two cases.\\  

\noindent\textbf{Case 1: $U$ is countable.}  
Without loss of generality, rewrite \eqref{pf:k} as
\begin{align}
k(x,y) = \sum_{i\in \mathbb{Z}} g_i(x) g_i(y). \label{pf:k_count}
\end{align}
For each $\omega\in\mathbb{R}$, let $I(\omega)=i$ be the unique integer such that $i-1<\omega\le i$. Define
$$
\phi_\omega(x) := \frac{1}{\sqrt{G((i-1,i])}} g_{I(\omega)}(x).
$$
Then
\begin{align}
\mathbb{E}_{\omega\sim G}[\phi_\omega(x)\phi_\omega(y)]
&= \sum_{i\in\mathbb{Z}} \frac{g_i(x) g_i(y)}{G((i-1,i])}  \mathbb{P}(\omega\in(i-1,i]) \nonumber\\
&= \sum_{i\in\mathbb{Z}} g_i(x) g_i(y) \nonumber\\
&= k(x,y). \nonumber
\end{align}

\noindent \textbf{Case 2: $U$ is uncountable.}  
Let $\nu$ be the uniform measure on the index set $\{\alpha : u_\alpha \in U\}$. Since $\nu$ is non-discrete, it is a continuous measure. Hence the Radon–Nikodym derivative
$$
\frac{d\nu}{dG}:\mathbb{R}\to \mathbb{R}
$$
is well defined. Extend $g_\alpha$ to all $\alpha\in\mathbb{R}$ by
$$
\bar{g}_\alpha(x) :=
\begin{cases}
g_\alpha(x), & \text{if } u_\alpha\in U,\\[4pt]
0, & \text{otherwise}.
\end{cases}
$$
Now define
$$
\phi_\omega(x) := \bar{g}_\omega(x)\sqrt{\tfrac{d\nu}{dG}(\omega)}.
$$
Then
\begin{align}
\mathbb{E}_{\omega\sim G}[\phi_\omega(x)\phi_\omega(y)]
&= \int_{\mathbb{R}} \bar{g}_\omega(x)\bar{g}_\omega(y)\frac{d\nu}{dG}(\omega) dG(\omega) \nonumber\\
&= \int_{\mathbb{R}} \bar{g}_\omega(x)\bar{g}_\omega(y) d\nu(\omega) \nonumber\\
&= k(x,y). \nonumber
\end{align}

\noindent Thus in both cases the kernel $k(x,y)$ admits the representation \eqref{eq:kernel}, which completes the proof.

\begin{remark}\label{rk:kernel}
In the proof, we can see each $k(x,y):\mathbb{R}^d\times \mathbb{R}^d\to \mathbb{R}$ can be written as 
$$\mathbb{E}_{\omega\sim\mathcal{N}(0,1)}[\phi(\omega,x)\phi(\omega,y)]$$
for some measurable function $\phi:\mathbb{R}\times \mathbb{R}^d\to \mathbb{R}$. 
\end{remark}

\section{Approximation error for learnable kernel.}

In practice, we utilize a finite layer MLP (e.g. the model described in section 2.2) to parametrize the feature mapping, denoted as $\phi^\theta$  and select a finite number of $\{\omega_i\}_{i=1}^m\sim \text{i.i.d.}\mathcal{N}(0,1)$ and let $G_m=\sum_{j=1}^m\delta_{\omega_j}$. 
The induced kernel becomes: 
$$k^{\theta}(x,y)=\mathbb{E}_{\omega\sim G_n}[\phi_{\omega}^\theta(x)\cdot\phi^\theta_{\omega}(y)].$$

\noindent Given a machine learning task, suppose $k^*(x,y)$ is an optimal kernel for this task, that is: 
\begin{align}
k^*(x,y)\in \arg\min_{k\text{ is a kernel}}\mathbb{E}_{Z\sim \mu_t}[\mathcal{L}_t(k,Z)]\label{eq:population_sol}
\end{align}
where $\mu_t$ is the ground truth distribution for task $t$.



\noindent One natural question is: \textit{What is the approximation error: $$|k^*(x,y)-k^{\theta}(x,y)|?$$}
The above error can be decomposed into three type of errors: parameterization error, generative error from the Gaussian sample $\omega_i$.

\subsection{Assumptions and Fundamental Results}\label{sec:param_error}
\begin{assumption}\label{asp:phi_1}
Let $K \subset \mathbb{R}^d$ be a compact set. We impose the following assumptions on $\phi$:
\begin{enumerate}
\item[(1)] For each $x \in \mathbb{R}^d$, the map $\omega \mapsto \phi(\omega, x)$ is $\|\phi(\cdot, x)\|_{\mathrm{lip}}$-Lipschitz.
where $$\|\phi(\cdot,x)\|_{lip}:=\sqrt{\sum_{i=1}^D \|\phi^i(\cdot,x)\|_{lip}^2}.$$
That is, we define the Lipschiz norm as the $L_2$ norm of the Lipschiz norm of each component. 
\item[(2)] For each $x \in \mathbb{R}^d$, the map $\omega \mapsto \phi(\omega, x)$ is bounded, i.e.,
$$
\sup_{\omega \in \mathbb{R}^d} \|\phi(\omega, x)\| < \infty.
$$
\item[(3)] $\phi$ is uniformly bounded when $x$ is restricted to $K$, i.e.,
$$
\max_{\omega \in \mathbb{R}^d,  x \in K} \|\phi(\omega, x)\| < \infty.
$$
\item[(4)] 
$\phi(\cdot,x)$ has a uniform Lipschitz norm, that is:
$$
\max_{x \in K} \|\phi(\cdot, x)\|_{\mathrm{lip}} < \infty.
$$
\item[(5)] We define 
$$
g(\omega) = \sup_{x \in K} |\phi(\omega, x)|^2.
$$
Then $g(\omega) < \infty$ for all $\omega$, and 
$$
\mathbb{E}_{\omega \sim \mathcal{N}(0, I_d)}[g(\omega)] < \infty.
$$
\item[(6)] The kernel can be expressed as
\begin{align*}
k(x,y)
&= h(x)\,k'(x-y)\,h(y) \\
&= \mathbb{E}_{\omega\sim\mathcal{N}(0,I_d)}
   \!\left[\phi(\omega,x)^\top \phi(\omega,y)\right].
\end{align*}
with $$\phi(\omega,x)=h(x)\psi(\omega^\top x);$$
where $h: \mathbb{R}^d \to \mathbb{R}$ and $k': \mathbb{R}^d \times \mathbb{R}^d \to \mathbb{R}$.
The second equality follows from lemma \ref{lem:phi_decompose}.
Moreover, $h(x)$ is assumed to be continuous. 
\item[(7)] Consider MLPs $\psi_{NN},h_{NN}$, we suppose their activation function $\sigma$ is Lipschiz, nonpolynomial function. 

\end{enumerate}
\end{assumption}

\subsection{Parametrization error in the learnable kernel.}

In this section, we define a feature mapping via the model \eqref{eq:phi-def-scalarenv}, i.e., $x\mapsto F(u)$ with $u=\omega^\top x$. Our goal is to analyze the expressive error $|k^\theta(x,y)-k(x,y)|$, where 
$k^\theta(x,y)=\mathbb{E}_{\omega\sim\mathcal{N}(0,I_d)}[\phi(\omega,x)\phi(\omega,y)].$

\paragraph{Properties the feature mapping.}We first define the feature mapping and introduce some fundamental properties of it. 

By Proposition \ref{prop:kernel} and Remark \ref{rk:kernel}, we have the following. 
\begin{lemma}\label{lem:feature_map}
For each $D\in \mathbb{N}$, 
there exists a mapping 
$$\mathbb{R}^d\times \mathbb{R}^d\ni(x,\omega)\to \phi(\omega,x)\in \mathbb{R}^{D}$$
such that  
$$k(x,y)=\mathbb{E}_{\omega\in \mathcal{N}(0,I_d)}[\phi(\omega,x)\cdot\phi(\omega,y)].$$
\end{lemma}
\begin{proof}
By Proposition \ref{prop:kernel}, there exists a mapping 
$$\mathbb{R}\times \mathbb{R}^d\ni (\omega_1,x)\mapsto \psi(\omega_1,x)\in \mathbb{R}$$
such that $$k(x,y)=\mathbb{E}_{\omega_1\in\mathcal{N}(0,1)}[\phi(\omega_1,x)\phi(\omega,y)],\forall x,y\in \mathbb{R}^d.$$

\noindent Define:
$$\phi(\omega,x)=[\psi(\omega_1,x);0_{D-1}]$$
and we have 
\begin{align}
&\mathbb{E}_{\omega\sim \mathcal{N}(0,I_d)}[\phi(\omega,x)\cdot\phi(\omega,x)]\nonumber\\
&=\mathbb{E}_{\omega_1,\ldots \omega _d\sim\text{i.i.d. }\mathcal{N}(0,1)}[\psi(\omega_1,x)\psi(\omega_1,y)]\nonumber\\
&=\mathbb{E}_{\omega_2,\ldots \omega_d\sim \text{i.i.d }\mathcal{N}(0,1)}[\mathbb{E}_{\omega_1\sim\mathcal{N}(0,1)}[\psi(\omega_1,x)\psi(\omega_1,y)]]\nonumber\\
&=\mathbb{E}_{\omega_2,\ldots \omega_d\sim \text{i.i.d }\mathcal{N}(0,1)}[k(x,y)]=k(x,y)\nonumber 
\end{align}
This completes the proof. 
\end{proof}

\begin{lemma}\label{lem:phi_decompose}
Suppose $k(x,y)$ satisfies Assumption \ref{asp:phi_1} (6), there exists feature mapping in the form. $$\phi(\omega,x)=h(x)\psi(\omega,x).$$    
\end{lemma}
\begin{proof}
We directly obtain the result from Bochner’s theorem \cite{rahimi2007random,rudin2017fourier} as discussed in the main text. Where $\psi(\omega^\top x)=[\cos(\omega^\top x),\sin(\omega^\top x)]$ or $\psi(\omega,x)=\sqrt{2}\cos(\omega^\top x+b)$.    
\end{proof}

\begin{lemma}\label{lem:kernel_compact}
Suppose the feature mapping $\phi(\omega,x)$ satisfies condition (5) in Assumption \ref{asp:phi_1}. Then there exists a sufficiently large $R$  such that,  
$$\mathbb{E}_{\omega}[\|\phi\|-\|\phi\|1_{A}]\leq \epsilon,$$
where $A=\{\omega: \|\omega\|\leq R\}$. 
\end{lemma}
\begin{proof}
Let $g(\omega)=\sup_{x\in K}\|\phi(\omega,x)\|^2$. By condition (5) in Assumption \ref{asp:phi_1}, we have $$\mathbb{E}_{\omega\sim\mathcal{N}(0,I_d)}[g(X)]<\infty.$$
Thus, by the dominated convergence theorem and the fact $g1_{B}\leq g$ for any set $B\subset \mathbb{R}^d$, we have 
\begin{align}
\mathbb{E}[g(\omega) 1_{\|\omega\|\ge R}]\searrow 0 \quad \text{as } R\to\infty. \nonumber 
\end{align}

\noindent Pick $R>0$ such that $\mathbb{E}[g(\omega)1_{|\omega|\ge R}]\leq \epsilon$. We then have 
\begin{align}
\mathbb{E}[\|\phi(\omega,x)\|^21_{\omega\ge R}]\leq \mathbb{E}[g 1_{|\omega|\ge R}]\leq \epsilon, 
\end{align}
and the proof is complete. 
\end{proof}

\begin{lemma}\label{lem:phi_compact}
Given a non-empty compact set $K$, 
suppose $k(x,y)$ satisfies Assumption~\ref{asp:phi_1}(6), 
so that we can write the feature mapping as
$\phi(\omega,x) = h(x)\,\psi(\omega^\top x)$.
Assume moreover that $\psi$ satisfies condition (5) in 
Assumption~\ref{asp:phi_1}.
Then there exists a compact set $A$ such that, for all $x,y \in K$,
\begin{equation}\label{eq:kA-approx}
\begin{aligned}
\bigl|k(x,y) - k^{A}(x,y)\bigr|
&\le \epsilon,\\
k^{A}(x,y)
&:= \mathbb{E}_{\omega \sim \mathcal{N}(0,I_d)}
\bigl[\phi(\omega,x)\,\phi(\omega,y)\,\mathbf{1}_{\{\omega\in A\}}\bigr].
\end{aligned}
\end{equation}
\end{lemma}
\begin{remark}

The condition (5) in Assumption \ref{asp:phi_1} is much weaker than the boundedness of $\psi(\omega,x)$. The two widely used $\psi$ functions derived in Bochner's Theorem are bounded and thus satisfy the condition (6). 
\end{remark}
\begin{proof}
From assumption \ref{asp:phi_1} (6), we have  $\phi(\omega,x)=h(x)\psi(\omega^\top x)$.
where $C$ is a constant. It is clear $C \|\omega\|$  is integrable in $\mathbb{R}^d,\mathcal{N}(0,I_d)$. Thus (5) in Assumption \ref{asp:phi_1} are satisfied. 

\noindent From Lemma \ref{lem:phi_compact}, there exists compact set $A$ such that $\mathbb{E}_{\omega\sim\mathcal{N}(0,I_d)}[\|\psi(\omega,x)\|1_{A^c}]\leq \epsilon,\forall x$. Then we have 
\begin{align}
&|k(x,y)-k^A(x,y)|\nonumber\\
&\leq \mathbb{E}_{\omega\sim\mathcal{N}(0,I_d)}[|h(x)\psi(\omega,x)-h(x)\psi(\omega,x) 1_{\omega\in A}|]\nonumber\\
&=|h(x)|\mathbb{E}_{\omega}[|\psi(\omega,x)|1_{A^c}]\nonumber\\
&\leq \max_{x\in K}|h(x)|\epsilon 
\end{align}
Since $x\mapsto h(x)$ is continuous, $\max_{x\in K}|h(x)|$ is finite. Thus we complete the proof. 
\end{proof}

\paragraph{Some useful Lemmas.} 
Now we introduce some fundamental lemma and inequalities. 
\begin{lemma}\label{lem:product_diff}
Pick $a_1,a_2,b_2,b_2\in \mathbb{R}^d$ (or $a_1,a_2\in \mathbb{R}, b_1,b_2\in \mathbb{R}^d$), let $\Delta_a=a_1-a_2,\Delta_b=b_1-b_2$, we have 
\begin{align}
&\|a_1\cdot b_1 -a_2\cdot b_2\| \nonumber\\
&=\|a_1\cdot b_1- (a_1-\Delta_a)\cdot (b_1-\Delta_b)\|\nonumber\\
&=\|\Delta a\cdot b+a_1\Delta_b-\Delta_a\cdot \Delta_b\|\nonumber\\
&\leq \|\Delta a\| \|b_1\|+ \|\Delta b\|\|a_1\|+ \|\Delta_a\|\|\Delta_b\|\nonumber 
\end{align}
\end{lemma}

\paragraph{Universial Approximation Theorem}
Next, we introduce the following classical Universal approximation theorem: 
\begin{theorem}\label{thm:uat_1}[\cite{cybenko1989approximation}]
Consider the following 1-hidden-layer MAP hypothesis class, 
where $\sigma$ is not a polynomial function:
\begin{equation}\label{eq:map-class}
\begin{aligned}
\mathcal{H}
&:=
\Bigl\{
A\sigma(Wx+b)
:
W\in\mathbb{R}^{w\times d},\ 
b\in\mathbb{R}^w,\ 
A\in\mathbb{R}^{D\times w}
\Bigr\},\\
&=
\overline{\operatorname{span}}\bigl\{\sigma(Wx+b)\bigr\}.
\end{aligned}
\end{equation}
Then $\mathcal{H}$ is dense in $\mathcal{C}(K)$ with respect to $\|\cdot\|_{\sup}$ 
for each compact set $K$, and it is dense in $\mathcal{C}(\mathbb{R}^d)$ with respect to 
$\|\cdot\|_{\mu}$, where $\mu$ is a finite positive measure supported on a compact set.
\end{theorem}

\noindent The compact domain assumption can be further relaxed: 
\begin{theorem}\cite{hornik1991approximation}\label{thm:uat_2}Under the same setting of the Theorem \ref{thm:uat_1}, suppose $\sigma$ is unbounded non-polynomial, we have 
$\mathcal{H}$ can is dense in $\mathcal{L}^p(\mu),\forall 1\leq p<\infty$ for all finite positive measure $\mu$ defined in $\mathbb{R}^d$. 
\end{theorem}

\begin{theorem}\label{thm:uat_3}[\cite{Ito1992ShiftedRotations,WangQu2022_approxUnbounded}]
The hypothesis class of all 1-hidden layer MLPs, $\mathcal{H}$ is dense in space 
$\mathcal{C}_0(\mathbb{R}^n)$ (set of continuous functions converging to 0 at infinity) and $\mathcal{C}(\bar{\mathbb{R}}^n)$ (set of continuous functions converging to a finite number at infinity).
\end{theorem}





\paragraph{Parametrization error Analysis}
Based on the above UAT theorems and lemmas, we will discuss the parametrization error.  

\begin{proposition}\label{pro:error_1_bounded}
Suppose $k(x,y)$ and its feature mapping satisfy the conditions in Lemma \ref{lem:kernel_compact}. In addition, suppose feature mapping $\phi$ is continuous. Set Hypothesis class 
$$\mathcal{H}=\{\phi_{NN}(\omega,x):=h_{NN}(x)\psi_{NN}(\omega^\top x)\}$$
where $h_{NN}:\mathbb{R}^d\to\mathbb{R},\psi_{NN}:\mathbb{R}\to \mathbb{R}^{D}$ are the 1-hidden layer MLPs. \\

\noindent Then there exists $M>0$ such that
\begin{equation}
\begin{aligned}
\phi_{NN} &\in \mathcal{H}^M, \\
\mathcal{H}^M 
&:= \{f(\omega,x)=h_{NN}(x)\psi_{NN}(\omega^\top x)^M\}, \\
f^M &:= \text{clip}(f,[-M,M]).
\end{aligned}
\end{equation}

\noindent such that 
$$|k(x,y)-k_{NN}(x,y)|\leq \epsilon,\forall x,y\in K,$$
where $k_{NN}(x,y)=\mathbb{E}_{\omega\sim\mathcal{N}(0,I_d)}[\phi_{NN}(\omega,x)\cdot \phi_{NN}(\omega,y)]$. 

\end{proposition}
\begin{proof}
Pick the compact set $A$ in lemma \ref{lem:phi_compact} such that 
$$\mathbb{E}[\|\phi(x,y)\|]\leq\epsilon,\forall x,y\in K.$$
Since $k(x,y)$ satisfies Assumption \ref{asp:phi_1} (6),  by lemma \ref{lem:phi_decompose}, there exists feature mapping $\phi(\omega,x)=h(x)\psi(\omega^\top x)$. Furthermore, we have 
$$M=\sup_{\omega\in A,x\in K}\|\phi(\omega,x)\|<\infty $$ 
since $\phi$ is continuous. \\ 

\noindent By the UAT Theorem \ref{thm:uat_1}, there exists one-hidden-layer MLP,
$\psi_{NN}$ such that 
$$\|\psi(\omega^\top x)-\psi_{NN}(\omega^\top x)\|\leq \epsilon,\forall (\omega,x)\in A\times K.$$ 

\noindent Let $\psi^M_{NN}=\text{Clip}(\psi_{NN},[-M, M])$. We have 
\begin{multline}
\sup_{\omega\in M,x\in K}
\|\psi(\omega^\top x)-\psi_{NN}^M(\omega^\top x)\|
\leq \\
\|\psi(\omega^\top x)-\psi_{NN}(\omega,x)\|
\leq \epsilon.
\end{multline}

\noindent Similarly, we can find $h_{NN}$ such that $\|h_{NN}-h\|\leq \epsilon,\forall x\in K$. 
Let $\phi_{NN}=h_{NN}\psi_{NN}^M$. 
All expectations below are with respect to $\omega\sim\mathcal{N}(0,I_d)$. 
We have
\begin{equation}
\begin{split}
k(x,y)-k_{NN}(x,y)
&= \mathbb{E}\bigl[\phi(\omega,x)\,\phi(\omega,y)\bigr] \\
&\quad - \mathbb{E}\bigl[\phi_{NN}(\omega,x)\,\phi_{NN}(\omega,y)\bigr].
\end{split}
\end{equation}
By Lemma~\ref{lem:product_diff},
\begin{equation}
\begin{aligned}
k(x,y)-k_{NN}(x,y) \le {}&
\mathbb{E}\bigl[\|\phi(\omega,x)\|\,\Delta(\omega,y)\bigr] \\
&+ \mathbb{E}\bigl[\|\phi(\omega,y)\|\,\Delta(\omega,x)\bigr] \\
&+ \mathbb{E}\bigl[\Delta(\omega,x)\,\Delta(\omega,y)\bigr].
\end{aligned}
\end{equation}

\noindent Therefore,
\begin{equation}\label{pf:k_to_bound}
\begin{aligned}
k(x,y)-k_{NN}(x,y)
&\leq \mathbb{E}^{1/2}[\Delta^2(\omega,x)]\cdot
      \mathbb{E}^{1/2}[\phi^2(\omega,y)] \\
&\quad+ \mathbb{E}^{1/2}[\Delta^2(\omega,y)]\cdot
        \mathbb{E}^{1/2}[\phi^2(\omega,x)] \\
&\quad+ \mathbb{E}^{1/2}[\Delta^2(\omega,x)]\cdot
        \mathbb{E}^{1/2}[\Delta^2(\omega,y)].
\end{aligned}
\end{equation}

\noindent where $\Delta(\omega,x)=\|\phi(\omega,x)-\phi_{NN}(\omega,x)\|$. We have:
\begin{align}
\mathbb{E}^{1/2}[\phi^2(\omega,x)]=k^{1/2}(x,x)\leq \max_{x\in K} k^{1/2}(x,x)\label{pf:phi^2}  
\end{align}
and the same property holds for $y$. It remains to bound 
$\sup_{x\in K}\mathbb{E}[\Delta^2(\omega,x)].$ We have 
\begin{equation}\label{pf:Delta_to_bound}
\begin{split}
\mathbb{E}[\Delta^2(\omega,x)]
&= \mathbb{E}\bigl[\|\phi(\omega,x)-\phi_{NN}(\omega,x)\|^2\bigr] \\
&\leq 
\underbrace{\mathbb{E}\bigl[\|\phi(\omega,x)-\phi_{NN}(\omega,x)\|^2 1_{\omega\in A}\bigr]}_{B_1}\\
&+
\underbrace{\mathbb{E}\bigl[\|\phi(\omega,x)\|^2 1_{\omega\notin A}\bigr]}_{B_2}
+
\underbrace{\mathbb{E}\bigl[\|\phi_{NN}(\omega,x)\|^2 1_{\omega\notin A}\bigr]}_{B_3}.
\end{split}
\end{equation}

\noindent Let $\Delta_h=h-h_{NN}, \delta_\psi=\psi-\psi_{NN}$, we have: 

By Lemma~\ref{lem:product_diff}, we have
\begin{equation}
\begin{split}
B_1 
&\leq \mathbb{E}_A\bigl[\||h|\|\Delta_\psi\|
      +|\Delta_h|\|\psi\|
      +|\delta_h|\|\Delta_\psi\|^2\bigr] \\
&\leq \mathbb{E}_A\bigl[(\max_{A}|h|\epsilon+M\epsilon+\epsilon^2)^2\bigr] \\
&= (\max_{A}|h|\epsilon+M\epsilon+\epsilon^2)^2 \\
&=: B_1(\epsilon)=\mathcal{O}(\epsilon^2).
\end{split}
\end{equation}

\noindent By Lemma~\ref{lem:phi_compact},
\begin{equation}
B_2 \leq \max_{A} h(x)\,\epsilon
=: B_2(\epsilon) = \mathcal{O}(\epsilon^2).
\end{equation}

\noindent From the definition of $\psi_{NN}^M$, and the fact $|h-h_{NN}|\leq \epsilon$, we have
\begin{equation}
\begin{split}
B_3
&\leq (\max_A|h(x)|+\epsilon)^2\, M\, \mathbb{P}(\omega\in A^c) \\
&\leq (\max_A|h(x)|+\epsilon)^2\, M\, \mathbb{P}\epsilon
=: B_3(\epsilon).
\end{split}
\end{equation}

\noindent Applying $B_1(\epsilon), B_2(\epsilon), B_3(\epsilon)$ to~\eqref{pf:Delta_to_bound}, we have

\begin{align}
    \mathbb{E}[\Delta^2(\omega,x)]\leq B_1(\epsilon)+B_2(\epsilon)+B_3(\epsilon)=:B(\epsilon)=\mathcal{O}(\epsilon).\label{pf:Delta_x}
\end{align}

\noindent Applying \eqref{pf:phi^2} and \eqref{pf:Delta_x} to \eqref{pf:k_to_bound}, we obtain 
\begin{align}
|k(x,y)-k^h(x,y)|\leq 2\sqrt{B(\epsilon)}\max_{x\in K}k^{1/2}(x,x)+B(\epsilon), \nonumber  
\end{align}
which completes the proof. 
\end{proof} 

\noindent There are two ways to relax the upper bound $M$ in the hypothesis setting. We start from the simple one: 
\begin{proposition}\label{pro:param_error_2} 
Under the same setting of Proposition \ref{pro:error_1_bounded}. We redefine the Hypothesis as: 
$$\mathcal{H}:=\left\{h_{NN}(\psi_{NN}+\sum_{i=1}^K\alpha_i \text{Relu}(\psi_{NN})):K\in \mathbb{N},\alpha_i\in \mathbb{R}\right\}$$
where $h_{NN}:\mathbb{R}^d\to \mathbb{R}$ is a 1-hidden layer MLP, $\psi_{NN}$ is a 1-hidden layer MLP. 

\noindent Then there exists $k_{NN}$ defined by feature mapping in $\mathcal{H}$ such that 
$$|k(x,y)-k_{NN}(x,y)|\leq \epsilon,\forall x,y\in K.$$
\end{proposition}
\begin{proof}
From Proposition \ref{pro:error_1_bounded}, there exists 1-hidden layer $h_{NN}$ and 1-hidden layer $\psi_{NN}$ such that the the kernel $k_{NN}$ induced by $h_{NN} (\text{clip}(\psi_{NN},[-M, M])$, can approximate $k$, i.e. 
$$|k(x,y)-K_{NN}(x,y)|\leq \epsilon,\forall x,y \in K.$$
Furthermore, we have 
\begin{align}
&\text{clip}(\psi_{NN},[-M,M])\nonumber\\
&= \psi_{NN} -ReLu(\psi_{NN}-M)+ReLu(-\psi_{NN}-M)
\nonumber \\
&=:\tilde{\psi}_{NN}
\end{align}

$$$$ Thus $h_{NN}\tilde{\psi}_{NN}\in\mathcal{H}$ and we complete the proof. 
\end{proof}

\begin{proposition}\label{pro:error1_unbounded}
Under the same setting of Proposition \ref{pro:error_1_bounded},
let $\mathcal{H}:=\{h_{NN}(x)\psi_{NN}(\omega^\top x)\}$, where $\psi_{NN}$ is 1-hidden layer MLP with unbounded, non-polynomial activation (e.g. ReLu), then there exists $\phi_{NN}\in \mathcal{H}$ such that $|k(x,y)-k_{NN}(x,y)|\leq \epsilon,\forall x,y\in K$.   
\end{proposition}
\begin{proof}
First, by lemma \ref{lem:phi_compact}, there exists a compact set $A_1$ such that 
$$\mathbb{E}_{\omega\sim 
\mathcal{N}(0,I_d)}[\|\psi 1_{A_1^c}(x,\omega)\|]\leq \epsilon,\forall x\in K.$$
Furthermore, there exists compact set $A_2$ such that $\mathbb{P}(\omega\in A_2)\leq \epsilon$. We set compact set $A=A_1\cup A_2$. 

\noindent We can define the following $C_0$ function $\psi_0:\mathbb{R}\to \mathbb{R}^{D}:$
Indeed, there exists $R\ge 0$ such that $\omega^\top x\in [-R,R],\forall (\omega,x)\in A\times K$. We define 
\begin{align}
(\omega,x)\mapsto \psi_0(\omega^\top x)= \begin{cases}
\psi(\omega^\top x) &\text{If } \omega^\top x\in [-R,R] \\
(1-(\omega^\top x-R))\psi(R) & \text{if }\omega^\top x\in [R,R+1] \\
(1+\omega^\top x+R)\psi(-R) & \text{if }\omega^\top x\in [-R-1,-R] \\
0   &\text{elses}
\end{cases}
\end{align}
We have: 
\begin{align}
\begin{cases}
\psi_0(\omega^\top x) =\psi (\omega^\top x), \quad\text{if }(\omega,x)\in A\times K\\
\|\psi_0(\omega^\top x)\|\leq \max(\|\psi(R)\|,\|\psi(\omega^\top x)\|), &\text{otherwise}\nonumber 
\end{cases} 
\end{align}

\noindent For the fase $(\omega,x)\notin A\times K$, we have 
$$\|\psi-\psi_0\|^2\leq \frac{1}{2}\|\psi\|^2+\frac{1}{2}\|\psi_0\|^2\leq \|\psi\|^2+\|\psi(R)\|^2$$

\noindent Thus, we have 
\begin{align}
&\mathbb{E}_{\omega\in \mathcal{N}(0,I_d)}[\|\psi-\psi_0\|^2]\nonumber\\
&=\mathbb{E}_{\omega\in \mathcal{N}(0,I_d)}[\|\psi-\psi_0\|^2 1_{\omega\in A^c}]\nonumber\\
&\leq 0+\mathbb{E}_{\omega\in \mathcal{N}(0,I_d)}[(\|\psi\|^2+\|\psi(R)\|^2) 1_{\omega\in A^c}]\nonumber\\
&\leq \epsilon+ \|\psi(R)\|^2 \epsilon=:B_1 \epsilon. \nonumber 
\end{align}

\noindent By UAT theorem \ref{thm:uat_3}, there exists 1-hidden layer MLP $\psi_{NN}$ such that 
$$\|\psi_{NN}(\omega^\top x)-\psi_0(\omega^\top x)\|\leq \epsilon,\forall (\omega,x)\in \mathbb{R}^d\times K.$$
Pick $h_{NN}$ that can uniformly approximate $h$ on $K$: 
$$\|h_{NN}-h\|\leq \epsilon,\forall x\in K.$$
\noindent Let
\begin{equation}
\begin{aligned}
\Delta(\omega,x) &= \|\phi(\omega,x)-\phi_{NN}(\omega,x)\|, \\
\Delta_h &= |h-h_{NN}|, \\
\Delta_{\psi_0} &= \|\psi_0-\psi_{NN}\|, \\
\Delta_{\psi} &= \|\psi-\psi_0\|.
\end{aligned}
\end{equation}
\noindent We have
\begin{equation}
\begin{aligned}
\mathbb{E}[\Delta^2(\omega,x)]
&= \mathbb{E}[\|h\psi-h_{NN}\psi_{NN}\|^2] \\
&\leq \mathbb{E}\bigl[
  (\|h\psi-h\psi_0\|
   +\|h\psi_0-h_{NN}\psi_{NN}\|)^2
\bigr] \\
&\leq \mathbb{E}\bigl[
  (|h|\Delta_\psi
   +|h|\Delta_{\psi_0}
   +\|\psi_0\|\Delta_h
   +\Delta_h\Delta_{\psi_0})^2
\bigr] \\
&\leq \mathbb{E}\bigl[
  (\max_K|h|\Delta_\psi
   +\max_K|h|\Delta_{\psi_0} \\
&\qquad\qquad
   +\max_{\mathbb{R}^d}\|\psi_0\|\Delta_h
   +\Delta_h\Delta_{\psi_0})^2
\bigr] \\
&\leq \mathbb{E}\bigl[
  (\max_K |h|\Delta_\psi
   +\max_K |h|\epsilon \\
&\qquad\qquad
   +\max_{\mathbb{R}^d}\|\psi_0\| \epsilon
   +\epsilon^2 )^2
\bigr],
\end{aligned}
\end{equation}
\noindent where we used Lemma~\ref{lem:product_diff}. By the Cauchy–Schwarz inequality,
\begin{equation}\label{pf:Delta_x_2}
\begin{aligned}
\mathbb{E}[\Delta^2(\omega,x)]
&\leq \mathbb{E}\Bigl[
  C(\max_{K}|h|,\max_{\mathbb{R}^d}\|\psi_0\|)
  \bigl( \Delta_\psi^2+\epsilon^2+\epsilon^4\bigr)
\Bigr] \\
&= C\bigl(\mathbb{E}[\Delta_\psi^2]+\epsilon^2+\epsilon^4\bigr) \\
&\leq C(B_1\epsilon +\epsilon^2+\epsilon^4)
=: B(\epsilon),
\end{aligned}
\end{equation}
where $C(\max_{K}|h|,\max_{\mathbb{R}^d}\|\psi_0\|)=C$ is a constant depending only on 
$\max_{K}|h|$ and $\max_{\mathbb{R}^d}\|\psi_0\|$.

Applying \eqref{pf:phi^2} and \eqref{pf:Delta_x_2} to \eqref{pf:k_to_bound}, we obtain 
\begin{align}
|k(x,y)-k^h(x,y)|\leq 2\sqrt{B(\epsilon)}\max_{x\in K}k^{1/2}(x,x)+B(\epsilon), \nonumber  
\end{align}
and complete the proof. 
\end{proof}








\subsection{Generalization Error Induced by $\omega_i$}\label{sec:error_omega}

Given a kernel 
\begin{equation}
k(x,y)
= \mathbb{E}_{\omega\sim G}\bigl[\phi_\omega(x)\cdot\phi_\omega(y)\bigr],
\end{equation}
and for i.i.d.\ samples $\Omega:=\{\omega_1,\ldots,\omega_m\}$, we define the empirical kernel
\begin{equation}
\begin{aligned}
k^{\Omega}(x,y)
&:= \mathbb{E}_{\omega\sim \hat{G}^m}
    \bigl[\langle\phi_\omega(x),\phi_\omega(y)\rangle\bigr] \\
&= \frac{1}{m}\sum_{j=1}^m
   \phi_{\omega_j}(x)\cdot \phi_{\omega_j}(y).
\end{aligned}
\end{equation}

\noindent where $\hat{G}^m=\frac{1}{m}\sum_{j=1}^m\delta_{\omega_j}$ is the empirical distribution based on sample $\Omega$. In this section, we will analyze the error $$|k(x,y)-k^{\Omega}(x,y)|.$$

\begin{proposition}\label{pro:model_lip}
Consider our model defined in \eqref{eq:phi-def-scalarenv}, suppose $h=h_{NN},\psi=\psi_{NN}$ are \textbf{fixed} 1-layer MLPs (including constant mapping), then the model $\phi(\omega,x)$ satisfies (1)(2)(3)(4) in Assumption \ref{asp:phi_1} with bounded  Lipschiz  activation (e.g. Sigmoid) (see condition (7) in assumption \ref{asp:phi_1}). \\

\noindent In addition, the feature mapping \eqref{eq:phi-def-scalarenv} with unbounded activation (e.g. ReLu) satisfies (1)(4) with unbounded Lipschiz activation (e.g. ReLu).  
\end{proposition}
\noindent To prove the statement, we first introduce the following Lemmas: 

\begin{lemma}[Basic Lipschitz rules]\label{lem:lip}
Let $X\subset\mathbb{R}^d$ and $Y\subset\mathbb{R}^D$ be nonempty. Then:\\
\begin{enumerate}\setlength{\itemsep}{0.2em}
\item[(1)] If $f_1:X\to Y$ and $f_2:Y\to Z$ are Lipschitz, then
$f_2\circ f_1$ is Lipschitz and
\[
\|f_2\circ f_1\|_{\mathrm{lip}}
\le \|f_2\|_{\mathrm{lip}}\;\|f_1\|_{\mathrm{lip}} .
\]

\item[(2)] If $f_1,f_2:X\to Y$ are Lipschitz, then $f_1+f_2$ is Lipschitz and
\[
\|f_1+f_2\|_{\mathrm{lip}}
\le \|f_1\|_{\mathrm{lip}}+\|f_2\|_{\mathrm{lip}} .
\]

\item[(3)] If $f_1,f_2:X\to Y$ are bounded Lipschitz, then
$f_1\odot f_2$ is Lipschitz.

\item[(4)] If $f_1,f_2:X\to Y$ are bounded Lipschitz, then
$f_1\cdot f_2$ is Lipschitz.
\end{enumerate}
\end{lemma}

\begin{lemma}\label{lem:lip}
Given nonempty sets $X\subset\mathbb{R}^d,Y\subset\mathbb{R}^{D}$, we have:\\ 
\begin{enumerate}
    \item[(1)] If $f_1:X\to Y,f_2:Y\to Z$, and $f_1,f_2$ are Lipschiz functions. 
    Then $f_2\circ f_1$ is a $\|f_1\|_{lip}\|f_2\|_{lip}$-Lipschiz function.
    \item[(2)]  If $f_1,f_2:X\to Y$ are Lipschiz functions, then $f_1+f_2$ are Lipschiz.
    \item[(3)]  If $f_1,f_2:X\to Y$ are bounded Lipschitz function, then 
    $f_1\odot f_2$ is Lipschitz function. 
    \item[(4)]
    If $f_1,f_2:X\to Y$ are Lipschitz bounded functions, then 
    $f_1\cdot f_2$ is a Lipschitz function. 
\end{enumerate}
\end{lemma}

\begin{proof}
Pick $x,x'\in X,y,y'\in Y$. 

\medskip\noindent
(1) In this case, we have
\begin{equation*}
\begin{aligned}
\bigl\|f_2\!\circ\! f_1(x)-f_2\!\circ\! f_1(x')\bigr\|
&\le \|f_2\|_{\mathrm{lip}}\;\|f_1(x)-f_1(x')\| \\
&\le \|f_2\|_{\mathrm{lip}}\;\|f_1\|_{\mathrm{lip}}\;\|x-x'\|.
\end{aligned}
\end{equation*}

\medskip\noindent
(2) In this case,
\begin{equation*}
\renewcommand{\arraystretch}{1.1}
\begin{array}{@{}l@{}}
\|f_1(x)+f_2(x)-(f_1(x')+f_2(x'))\| \\[0.15em]
\quad\le \|f_1(x)-f_1(x')\|+\|f_2(x)-f_2(x')\| \\[0.15em]
\quad\le (\|f_1\|_{\mathrm{lip}}+\|f_2\|_{\mathrm{lip}})\,\|x-x'\|.
\end{array}
\end{equation*}

\medskip\noindent
(3) In this case,
\begin{equation*}\
\renewcommand{\arraystretch}{1.1}
\begin{array}{@{}l@{}}
\|f_1(x)\odot f_2(y)-f_1(x')\odot f_2(y')\| \\[0.15em]
\quad\le \|f_1(x)\|\,\|f_2\|_{\mathrm{lip}}\,\|y-y'\|
      + \|f_2(y')\|\,\|f_1\|_{\mathrm{lip}}\,\|x-x'\| \\[0.15em]
\quad\le \|f_1\|_{\infty}\,\|f_2\|_{\mathrm{lip}}\,\|y-y'\|
      + \|f_2\|_{\infty}\,\|f_1\|_{\mathrm{lip}}\,\|x-x'\| \\[0.15em]
\quad\le M_{\infty} M_{\mathrm{lip}}\,
        (\|x-x'\|+\|y-y'\|) \\[0.15em]
\quad\le \sqrt{2}\,M_{\infty} M_{\mathrm{lip}}\,
        \|(x;y)-(x';y')\|.
\end{array}
\end{equation*}
\end{proof}

\medskip\noindent
(4) We have
\begin{equation*}
\renewcommand{\arraystretch}{1.1}
\begin{array}{@{}l@{}}
\|f_1(x)\cdot f_2(x)-f_1(x')\cdot f_2(x')\| \\[0.15em]
\quad\le \|f_1(x)\cdot (f_2(x)-f_2(x'))\|
      + \|(f_1(x)-f_1(x'))\cdot f_2(x')\| \\[0.15em]
\quad\le \|f_1(x)\|\,\|f_2\|_{\mathrm{lip}}\,\|x-x'\|
      + \|f_2(x')\|\,\|f_1\|_{\mathrm{lip}}\,\|x-x'\| \\[0.15em]
\quad\le M_{\infty} M_{\mathrm{lip}}\,\|x-x'\|.
\end{array}
\end{equation*}

\begin{proof}[Proof of Proposition \ref{pro:model_lip}]
We first consider the bounded  activation. 
\begin{itemize}
    \item Conditions (2)(3). Since $\sigma(x)$ is a bounded function, thus we have $\psi,h$ are bounded function. Then $\phi$ satisfies (2)(3).  
    \item Condition (1). From the above lemma \ref{lem:lip}, we have the following mapping are Lipschitz function: 
\begin{align}
&x\mapsto u:=x^\top \omega; \nonumber\\
&\omega\mapsto u:=x^\top \omega \nonumber \\
&z=u\mapsto  \sigma(Wu) \nonumber; \nonumber\\
&z\mapsto Wz \nonumber 
\end{align}
Thus, $x\mapsto \psi(\omega^\top x)$ is Lipschitz for each $\omega$; 
$\omega\mapsto \psi(\omega^\top x)$ is Lipschitz for each fixed $x$. 

Similarly, $x\mapsto h_{NN}(x)$ is Lipschiz. 

\item 

It remains to prove the last condition (4) when $\sigma$ is bounded. First, we have: 
\begin{align}
\|\frac{du}{d\omega}\|\leq \|x\|\leq \max_{x\in K}\|x\|.\nonumber 
\end{align} 
Thus, define the finite constants
\begin{equation*}
\renewcommand{\arraystretch}{1.1}
\begin{array}{@{}l@{}}
M_K \coloneqq \max_{z\in K}\|z\|, \\[0.15em]
M_\psi \coloneqq \max_{i\in[1:D]}\|\psi^i\|_{\mathrm{lip}}, \\[0.15em]
M_h \coloneqq \max_{z\in K}|h(z)|.
\end{array}
\end{equation*}

For any $x\in K$, the map $\omega\mapsto \omega^\top x$ is $\|x\|$-Lipschitz, hence
\begin{equation*}
\renewcommand{\arraystretch}{1.1}
\begin{array}{@{}l@{}}
\|\omega\mapsto \omega^\top x\|_{\mathrm{lip}}
= \|x\| \le M_K .
\end{array}
\end{equation*}

Therefore, for each $i\in[1:D]$ and $x\in K$,
\begin{equation*}
\renewcommand{\arraystretch}{1.1}
\begin{array}{@{}l@{}}
\|\omega\mapsto \psi^i(\omega^\top x)\|_{\mathrm{lip}}
\le \|\psi^i\|_{\mathrm{lip}}\,
    \|\omega\mapsto \omega^\top x\|_{\mathrm{lip}} \\[0.15em]
\qquad\le \|\psi^i\|_{\mathrm{lip}}\,\|x\|
\le M_\psi\,M_K .
\end{array}
\end{equation*}

Taking the maximum over coordinates gives
\begin{equation*}
\renewcommand{\arraystretch}{1.1}
\begin{array}{@{}l@{}}
\|\psi(\cdot^\top x)\|_{\mathrm{lip}}
\le M_\psi\,M_K < \infty,
\qquad \forall\,x\in K .
\end{array}
\end{equation*}

Finally, with $\phi(\omega,x)=h(x)\psi(\omega^\top x)$,
\begin{equation*}
\renewcommand{\arraystretch}{1.1}
\begin{array}{@{}l@{}}
\|\phi(\cdot,x)\|_{\mathrm{lip}}
=\|h(x)\psi(\cdot^\top x)\|_{\mathrm{lip}}
\le |h(x)|\,\|\psi(\cdot^\top x)\|_{\mathrm{lip}} \\[0.15em]
\qquad\le |h(x)|\,M_\psi M_K
\le M_h\,M_\psi M_K < \infty,
\qquad \forall\,x\in K .
\end{array}
\end{equation*}

where the last inequality holds from the fact $h$ is continuous MLP, and thus $\max_{x\in K}\|h(x)\|<\infty$ and we prove (4). 



\end{itemize}

\end{proof}

\begin{proposition}[Gaussian concentration]\label{pro:Gaussian_concentration}
Suppose $f:\mathbb{R}^d\to \mathbb{R}$ is $L$-Lipschitz for some $L\ge 0$. 
Let $\Omega=\{\omega_1,\ldots,\omega_m\}$ be i.i.d.\ samples from $G=\mathcal{N}(0,I_d)$, and set
$$
\Delta_m=\mathbb{E}_{\omega\sim G}[f(\omega)]-\frac{1}{m}\sum_{j=1}^m f(\omega_j).
$$
Then:
\begin{itemize}
\item \textbf{Gaussian concentration inequality:}
$$
\mathbb{P}(|\Delta_m|\leq \epsilon)\ge 1-2\exp\!\left(-\frac{m\epsilon^2}{2L^2}\right).
$$
\item \textbf{Gaussian Poincaré inequality:}
\begin{align}
\mathbb{E}[\Delta_m^2]
= \frac{1}{m}\mathrm{Var}(f(\omega))
\leq \frac{1}{m}\mathbb{E}_{\omega\sim G}[\|\nabla f(\omega)\|_2^2]
\leq \frac{L^2}{m}. \nonumber
\end{align}
\end{itemize}
\end{proposition}

\begin{proof}
The first statement is the classical Gaussian concentration inequality.  
For the second, the variance bound follows from the Gaussian Poincaré inequality, and the final inequality uses the fact that $\|\nabla f\|_2\le L$ for an $L$-Lipschitz function. 
\end{proof}



\begin{lemma}\label{lem:kernel_lipschiz2}
Let $\phi:\mathbb{R}^d\times \mathbb{R}^d\to \mathbb{R}^H$ satisfy conditions (5) in    (2.1) in Assumption~\ref{asp:phi_1}.  
Then the mapping
$$\mathbb{R}^d\ni\omega\mapsto\langle \phi(\omega,x),\phi(\omega,y)\rangle $$
is $L$-Lipschitz with
$$
L=\max\{\|\phi(\cdot,x)\|,\|\phi(\cdot,y)\|\}\max\{L^2_x,L^2_y\}.
$$
\end{lemma}

Combining the above proposition and lemma \ref{lem:lip}, we obtain: 
\begin{proposition}\label{pro:error_omega_1}
Fix $x,y\in \mathbb{R}^d$ and suppose $\phi$ satisfies conditions (1)(2) in Assumption~\ref{asp:phi_1}. 
Let $\Omega=\{\omega_i\}_{i=1}^m\sim G$ be i.i.d.\ samples, and let $\hat{G}^m$ denote the empirical distribution. 
Define $\Delta_\Omega = k(x,y)-k^\Omega(x,y)$. Then:
\begin{align}
&\mathbb{P}(|\Delta_\Omega|\leq \epsilon)\ge 1-2\exp\!\left(-\frac{m\epsilon^2}{2L^2}\right), \label{eq:concentration_omega_1}\\
&\mathbb{E}[|\Delta_\Omega|^2]\leq \frac{L(x,y)}{m}, \label{eq:concentration_omega_2}
\end{align}
where
\begin{equation*}
\renewcommand{\arraystretch}{1.1}
\begin{array}{@{}l@{}}
L(x,y)\coloneqq
\max\!\bigl\{\|\phi(\cdot,x)\|_\infty,\ \|\phi(\cdot,y)\|_\infty\bigr\} \\[0.15em]
\qquad\quad\times
\max\!\bigl\{\|\phi(\cdot,x)\|_{\mathrm{lip}},\ \|\phi(\cdot,y)\|_{\mathrm{lip}}\bigr\}.
\end{array}
\end{equation*}
\end{proposition}

\begin{proof}
By Lemma~\ref{lem:lip}, the mapping 
$$\omega\mapsto \langle \phi(\omega,x),\phi(\omega,y)\rangle$$ 
is $L$-Lipschitz, where $L$ is defined in the above statement.  
Applying Proposition~\ref{pro:Gaussian_concentration} completes the proof. 
\end{proof}

\begin{corollary}\label{cor:err2_bound}
Let $K\subset\mathbb{R}^d$ be compact and suppose $\phi$ satisfies conditions (1)–(4) in Assumption~\ref{asp:phi_1}. 
Then for each $(x,y)\in K\times K$, inequalities \eqref{eq:concentration_omega_1}–\eqref{eq:concentration_omega_2} hold with
$$
L(K)=\max_{x\in K,\;\omega\in \mathbb{R}^d}\|\phi(\omega,x)\|\;\max_{x}\|\phi(\cdot,x)\|_{lip}^2.
$$
\end{corollary}

\begin{proof}
By (3)–(4) in Assumption~\ref{asp:phi_1}, $L(K)$ is finite. 
Moreover,
$$
L(K)=\sup_{x,y\in K} L(x,y)
$$
$L(x,y)$ is defined in proposition \ref{pro:error_omega_1}. 
Combining this with the previous proposition yields the result. 
\end{proof}
\begin{remark}\label{rm:error2_bound}
If the activation function is a bounded (e.g., sigmoid) function, combining the above corollary and Proposition \ref{pro:model_lip}, we have that our model \eqref{eq:phi-def-scalarenv} satisfies the above corollary. That is
\begin{align}
\mathbb{P}(|k(x,y)-\hat{k}(x,y)|\ge \epsilon)\leq 2\exp(\frac{-m\epsilon^2}{2L^2(K)})\nonumber 
\end{align}

\end{remark}

Now we consider the case activation function is unbounded (e.g., ReLu). We first recall the definition and related concepts in the sub-exponential distribution.

\begin{definition}
Given A random variable $X$. We consider the Orlicz norms
$$\|X\|_{\psi_p}=\inf\{K> 0:\mathbb{E}\left[ \exp\left(\frac{|X|^p}{K^p}\right)\right]\leq 2\}.$$

$X$ is said to follow a sub-exponential distribution if the following equivalent conditions holds: 
\begin{itemize}
    \item $\|X\|_{\psi_1}<\infty$, where $\|\cdot\|_{\psi_1}$ is called sub-exponential norm. 
    \item There exists a constant $c>0$ such that 
    $$\mathbb{P}(|X|>\epsilon)\leq 2\exp(-c \min (\frac{\epsilon^2}{\|X\|^2_{\psi_1}},\frac{\epsilon}{\|X\|_{\psi_1}})).$$
\end{itemize}
\end{definition}

We first introduce the following concentration inequality for product of two Lipschiz functions. Note, the challange is, if $f_1,f_2$ are Lipschiz but unbounded, $f_1f_2$ is not a Lipschitz function.  
\begin{proposition}\label{pro:product_concentration}[Exponential concentration.]
Consider two Lipschitz functions $f_1,f_2:\mathbb{R}^d\to \mathbb{R}$, for each $\epsilon\in (0,1)$, we have: 
\begin{equation}
\renewcommand{\arraystretch}{1.1}
\begin{array}{@{}l@{}}
\mathbb{P}_{\omega\sim\mathcal{N}(0,I_d)}
\!\left(
\bigl|f_1(\omega)f_2(\omega)
-\mathbb{E}[f_1(\omega)f_2(\omega)]\bigr|
\ge \epsilon
\right) \\[0.25em]
\quad\le
2\exp\!\left(
-c\,\min\!\left\{
\dfrac{\epsilon}{\|f_1\|_{\mathrm{lip}}\|f_2\|_{\mathrm{lip}}},\ 
\dfrac{\epsilon^2}{\|f_1\|_{\mathrm{lip}}^2\|f_2\|_{\mathrm{lip}}^2}
\right\}
\right).
\end{array}
\end{equation}
\end{proposition}
\begin{proof}
Since $\omega$ is Standard Gaussian, we have $\|\omega\|_{\psi_2}=1$. 
By the fundemental results for Sub-Gaussian  distribution,  we have 
$$\|f_1(\omega)\|_{\psi_2}\leq \|f_1\|_{lip},\|f_2(\omega)\|_{\psi_2}\leq \|f_2\|_{lip}.$$
From \citep[Lemma 2.8.6]{vershynin2018high}, we have 
and 
$$\|f_1(\omega)f_2(\omega)\|_{\psi_1}\leq \|f_1\|_{\psi_2}\|f_2\|_{\psi_2}\leq \|f_1\|_{lip}\|f_2\|_{lip}$$

Thus we complete the proof.  
\end{proof}

Next we extend the above argument to vector-valued functions. We start with a
basic property of the sub-exponential Orlicz norm.

\begin{lemma}\label{lem:psi_1_product}
Let $X$ and $Y$ be sub-exponential random variables. Then
\begin{equation*}
\renewcommand{\arraystretch}{1.1}
\begin{array}{@{}l@{}}
\|X+Y\|_{\psi_1}\le \|X\|_{\psi_1}+\|Y\|_{\psi_1}.
\end{array}
\end{equation*}
\end{lemma}

\begin{proof}
Set $a\coloneqq \|X\|_{\psi_1}$ and $b\coloneqq \|Y\|_{\psi_1}$, and let
\begin{equation*}
\renewcommand{\arraystretch}{1.1}
\begin{array}{@{}l@{}}
\alpha \coloneqq \dfrac{a}{a+b}, \qquad
p\coloneqq 1/\alpha, \qquad
p^*\coloneqq 1/(1-\alpha).
\end{array}
\end{equation*}
Using $|X+Y|\le |X|+|Y|$ and H\"older's inequality, we obtain
\begin{equation*}
\renewcommand{\arraystretch}{1.1}
\begin{array}{@{}l@{}}
\mathbb{E}[\exp\!\Bigl(\dfrac{|X+Y|}{a+b}\Bigr)]
\le
\mathbb{E}[\exp\!\Bigl(\dfrac{|X|+|Y|}{a+b}\Bigr)] \\[0.2em]
\qquad=
\mathbb{E}\!\left[
  \exp\!\Bigl(\dfrac{|X|}{a}\Bigr)^{\alpha}\,
  \exp\!\Bigl(\dfrac{|Y|}{b}\Bigr)^{1-\alpha}
\right] \\[0.2em]
\qquad\le
\mathbb{E}\!\left[
  \exp\!\Bigl(\dfrac{|X|}{a}\Bigr)^{\alpha p}
\right]^{1/p} \\[0.2em]
\qquad\quad\times
\mathbb{E}\!\left[
  \exp\!\Bigl(\dfrac{|Y|}{b}\Bigr)^{(1-\alpha)p^*}
\right]^{1/p^*} \\[0.2em]
\qquad\le 2^{\alpha}\,2^{1-\alpha}=2 .
\end{array}
\end{equation*}

By the definition of $\|\cdot\|_{\psi_1}$, this implies
$\|X+Y\|_{\psi_1}\le a+b$.
\end{proof}

\begin{proposition}\label{pro:exp-f1f2}
Let $f_1=[f_1^i]_{i=1}^D$ and $f_2=[f_2^i]_{i=1}^D$ be Lipschitz maps
$f_1,f_2:\mathbb{R}^d\to\mathbb{R}^D$.
Define the vector Lipschitz seminorms
\begin{equation*}
\renewcommand{\arraystretch}{1.1}
\begin{array}{@{}l@{}}
\|f_1\|_{\mathrm{lip}}
\coloneqq
\sqrt{\displaystyle\sum_{i=1}^D \|f_1^i\|_{\mathrm{lip}}^{2}},
\qquad
\|f_2\|_{\mathrm{lip}}
\coloneqq
\sqrt{\displaystyle\sum_{i=1}^D \|f_2^i\|_{\mathrm{lip}}^{2}} .
\end{array}
\end{equation*}
Then the scalar random variable
$g(\omega)\coloneqq f_1(\omega)\cdot f_2(\omega)
= \sum_{i=1}^D f_1^i(\omega)f_2^i(\omega)$
is sub-exponential and
\begin{equation*}
\renewcommand{\arraystretch}{1.1}
\begin{array}{@{}l@{}}
\|g\|_{\psi_1}
\le
\|f_1\|_{\mathrm{lip}}\;\|f_2\|_{\mathrm{lip}} .
\end{array}
\end{equation*}
\end{proposition}

\begin{proof}
For each coordinate, $f_1^i(\omega)$ and $f_2^i(\omega)$ are Lipschitz
functions of a Gaussian vector, hence sub-Gaussian; moreover,
\begin{equation*}
\renewcommand{\arraystretch}{1.1}
\begin{array}{@{}l@{}}
\|f_1^i(\omega)\|_{\psi_2}\le \|f_1^i\|_{\mathrm{lip}},
\qquad
\|f_2^i(\omega)\|_{\psi_2}\le \|f_2^i\|_{\mathrm{lip}} .
\end{array}
\end{equation*}
By \citep[Lemma 2.8.6]{vershynin2018high}, the product is sub-exponential:
\begin{equation*}
\renewcommand{\arraystretch}{1.1}
\begin{array}{@{}l@{}}
\|f_1^i(\omega)f_2^i(\omega)\|_{\psi_1}
\le
\|f_1^i\|_{\mathrm{lip}}\;\|f_2^i\|_{\mathrm{lip}} .
\end{array}
\end{equation*}
Using Lemma~\ref{lem:psi_1_product} repeatedly and then Cauchy--Schwarz,
\begin{equation*}
\renewcommand{\arraystretch}{1.1}
\begin{array}{@{}l@{}}
\|g\|_{\psi_1}
=
\Bigl\|\sum_{i=1}^D f_1^i(\omega)f_2^i(\omega)\Bigr\|_{\psi_1}
\le
\sum_{i=1}^D \|f_1^i(\omega)f_2^i(\omega)\|_{\psi_1} \\[0.15em]
\quad\le
\sum_{i=1}^D \|f_1^i\|_{\mathrm{lip}}\;\|f_2^i\|_{\mathrm{lip}}
\le
\sqrt{\sum_{i=1}^D \|f_1^i\|_{\mathrm{lip}}^{2}}\;
\sqrt{\sum_{i=1}^D \|f_2^i\|_{\mathrm{lip}}^{2}} \\[0.15em]
\quad=
\|f_1\|_{\mathrm{lip}}\;\|f_2\|_{\mathrm{lip}} .
\end{array}
\end{equation*}
\end{proof}

\begin{corollary}\label{cor:product_concentration}
Under the same conditions as Proposition~\ref{pro:product_concentration},
let $\{\omega_i\}_{i=1}^m$ be i.i.d.\ samples. Then
\begin{equation}
\renewcommand{\arraystretch}{1.1}
\begin{array}{@{}l@{}}
\mathbb{P}\!\left(
\left\|
\frac{1}{m}\sum_{i=1}^m f_1(\omega_i)f_2(\omega_i)
-\mathbb{E}[f_1(\omega)f_2(\omega)]
\right\|
\ge \epsilon
\right) \\[0.25em]
\quad\le
2\exp\!\left(
-c\,\min\!\left\{
\dfrac{\epsilon^2 m}{\|f_1\|_{\mathrm{lip}}^2\|f_2\|_{\mathrm{lip}}^2},\ 
\dfrac{\epsilon m}{\|f_1\|_{\mathrm{lip}}\|f_2\|_{\mathrm{lip}}}
\right\}
\right).
\end{array}
\end{equation}
\end{corollary}

\begin{proof}
By Proposition~\ref{pro:exp-f1f2}, $g(\omega)=f_1(\omega)\cdot f_2(\omega)$
is sub-exponential with
$\|g\|_{\psi_1}\le \|f_1\|_{\mathrm{lip}}\|f_2\|_{\mathrm{lip}}$.
Applying the sub-exponential Bernstein inequality
\citep[Theorem 2.9.1]{vershynin2018high} gives
\begin{equation*}
\renewcommand{\arraystretch}{1.1}
\begin{array}{@{}l@{}}
\mathbb{P}\!\left(
\left|
\frac{1}{m}\sum_{i=1}^m g(\omega_i)-\mathbb{E}[g(\omega)]
\right|
\ge \epsilon
\right) \\[0.15em]
\quad\le
2\exp\!\left(
-c\,\min\!\left\{
\dfrac{\epsilon^2 m}{\|g\|_{\psi_1}^{2}},\
\dfrac{\epsilon m}{\|g\|_{\psi_1}}
\right\}
\right) \\[0.15em]
\quad\le
2\exp\!\left(
-c\,\min\!\left\{
\dfrac{\epsilon^2 m}{\|f_1\|_{\mathrm{lip}}^2\|f_2\|_{\mathrm{lip}}^2},\ 
\dfrac{\epsilon m}{\|f_1\|_{\mathrm{lip}}\|f_2\|_{\mathrm{lip}}}
\right\}
\right).
\end{array}
\end{equation*}
\end{proof}

\begin{corollary}\label{cor:error2_unbounded}
Choose a compact set $K$ and restrict $x,y\in K$.
Assume the activation in~\eqref{eq:phi-def-scalarenv} is ReLU.
Then the finite kernel $\hat{k}$ satisfies the concentration bound in
Corollary~\ref{cor:product_concentration}. In particular,
\begin{equation}
\renewcommand{\arraystretch}{1.1}
\begin{array}{@{}l@{}}
\mathbb{P}\!\left(\|\hat{k}(x,y)-k(x,y)\|\ge \epsilon\right) \\[0.2em]
\quad\le
2\exp\!\left(
-c\,\min\!\left\{
\dfrac{m\epsilon}{L^2(K)},\ 
\dfrac{m\epsilon}{L(K)}
\right\}
\right), \\[0.2em]
\qquad \forall\,x,y\in K .
\end{array}
\end{equation}

where
\begin{equation*}
\renewcommand{\arraystretch}{1.1}
\begin{array}{@{}l@{}}
L(K)\coloneqq \max_{x\in K}\|\phi_{NN}(\cdot,x)\|_{\mathrm{lip}} \\[0.15em]
\qquad=
\max_{x\in K}\left(
  \sqrt{\displaystyle\sum_{i=1}^D
    \|\phi_{NN}^i(\cdot,x)\|_{\mathrm{lip}}^{2}}
\right).
\end{array}
\end{equation*}
\end{corollary}

\begin{proof}
By Proposition~\ref{pro:model_lip}, for each $x\in K$ the map
$\omega\mapsto \phi_{NN}(\omega,x)$ is Lipschitz with finite constant.
Hence $L(K)<\infty$, and the claim follows by applying
Corollary~\ref{cor:product_concentration}.
\end{proof}




\noindent Thus $L(K)$ is a finite number. From the fact $\phi(\cdot,x)$ satisfies (1), by Corollary \ref{cor:product_concentration}, we obtain the conclusion. 

\section{Implementation Details}
\subsection{Long Range Arena (LRA)}
\label{sec:lra_impl}

We evaluate on the Long Range Arena (LRA) benchmark, a standardized suite of
long-context classification tasks with sequence lengths between
$1\text{K}$ and $16\text{K}$ tokens. LRA is widely used to benchmark the
accuracy--efficiency trade-offs of Transformer variants under a consistent
evaluation protocol.

\paragraph{Codebases and preprocessing.}
All LRA experiments are conducted using the official benchmark repository,
including its preprocessing pipelines, tokenization/feature construction, and
fixed train/validation/test splits.\footnote{\href{https://github.com/google-research/long-range-arena}{google-research/long-range-arena}}
We treat this codebase as the source of truth for data handling and evaluation, and do not modify any of its processing or metrics.
For the Skyformer baseline, we rely on the authors' official implementation
and their released LRA training scripts.\footnote{\href{https://github.com/pkuzengqi/Skyformer}{pkuzengqi/Skyformer}}

\paragraph{Training protocol.}
For each LRA task, we follow the optimizer, learning-rate schedule,
regularization, and batching specified in the official LRA configuration files. Our method is implemented as a drop-in replacement for the attention module; all other architectural components and task-specific settings are kept identical to the corresponding baselines to ensure a fair comparison.

\subsection{Two-Stage $\phi$ Conversion on GLUE}
\label{sec:glue_impl}

We use a two-stage conversion pipeline to swap the standard softmax attention in a pretrained BERT model with our attention mechanism, while preserving the original query/key/value projections. The full procedure is implemented in our training script.

\paragraph{Codebase and tasks.}
All experiments are built on the HuggingFace \texttt{transformers} and
\texttt{datasets} stacks, using the official GLUE splits and evaluation metrics.The script supports the standard GLUE classification tasks and STS-B regression, and follows the same tokenization and preprocessing used by the corresponding pretrained checkpoints. 

\paragraph{Attention replacement.}
Starting from a pretrained BERT checkpoint, we replace every self-attention
module with our attention block. The swap is done in a drop-in manner: the
pretrained linear projections for $Q$, $K$, and $V$ are kept intact, and only
the attention computation is changed to our formulation. Any additional
parameters specific to our method (e.g., $\phi$) are instantiated per layer and head following the model’s original attention geometry. 

\paragraph{Stage 1: attention distillation.}
In the first stage, we freeze all pretrained parameters and train only the
newly introduced components (our $\phi$ parameters). The objective is to match the teacher attention distribution from the original softmax attention with the student distribution induced by our attention mechanism. This stage uses AdamW with a cosine schedule and a higher learning rate, and is run for a short warm-start (one epoch by default). 

\paragraph{Stage 2: task finetuning.}
After distillation, we unfreeze the full model and finetune end-to-end on the
GLUE task loss. We follow the standard supervised setup used in our appendix:
AdamW, batch size $8$, learning rate $10^{-5}$, zero weight decay, up to $10$
epochs, and early stopping with patience $3$ based on the validation metric of each task. 

\subsection{ViT-B/16 ImageNet-1K Conversion}
\label{sec:vit_imagenet_impl}

We perform a two-stage conversion of a pretrained ViT-B/16 model on
ImageNet-1K, replacing the original softmax self-attention with our attention
mechanism while keeping the pretrained patch embedding, positional embeddings, and $Q/K/V$ projections unchanged. The goal is to isolate the impact of the attention replacement under a standard large-scale vision training protocol.

\paragraph{Codebase and data.}
All experiments use a standard ViT training stack (timm-style
implementations) and the official ImageNet-1K train/validation splits.
Images are resized and center-cropped for evaluation, and trained with common
ImageNet augmentations (random resized crop and horizontal flip), optionally
combined with stronger regularization such as RandAugment, Mixup/CutMix, and
label smoothing to match typical ViT finetuning practice.

\paragraph{Attention replacement.}
Starting from a ViT-B/16 checkpoint pretrained with softmax attention, we
replace each multi-head self-attention block with our attention module in a
drop-in fashion. The pretrained linear layers producing $Q$, $K$, and $V$ are
reused without modification; only the attention computation is swapped.
Any additional parameters introduced by our method (e.g., $\phi$-related
parameters) are initialized per layer/head following the model’s original head structure.

\paragraph{Stage 1: attention distillation.}
In the first stage, we freeze all pretrained ViT parameters and train only the new components introduced by our attention mechanism. We distill attention by matching the teacher attention distribution from the original softmax module to the student distribution produced by our method, using the same image batches as the downstream task. This warm-start stage is run for a short schedule (typically a fraction of an epoch to one epoch on ImageNet-1K) with a higherlearning rate for the new parameters, and uses AdamW with a cosine decay schedule and small or no warmup.

\paragraph{Stage 2: supervised finetuning.}
After distillation, we unfreeze the entire network and finetune end-to-end on
ImageNet-1K with a standard ViT recipe. We use AdamW, cosine learning-rate
decay with warmup, weight decay in the usual ViT range, and mini-batch sizes
consistent with the pretrained setup. Regularization and augmentation choices
are kept identical across baselines and our converted model to ensure a fair
comparison.

\begin{lstlisting}[language=Python,caption={Our MLP-based feature map and linear attention.},label={lst:mlp_linattn}]

class TaskSpecificProjections(nn.Module):
    # shared W_i^T x + b_i over tokens/heads
    def __init__(self, d: int, M: int):
        super().__init__()
        self.M = int(M)
        self.W = nn.Parameter(torch.randn(self.M, d) / math.sqrt(d))
        self.b = nn.Parameter(torch.zeros(self.M))

    def forward(self, x):  # (B,H,N,d) -> (B,H,N,M)
        return torch.einsum("md,bhnd->bhnm", self.W, x) + self.b

class ScalarMLP(nn.Module):
    # maps scalar u to L channels in one pass (L=1 gives OneMLP)
    def __init__(self, L=1, hidden=64, act="relu", nonneg=True):
        super().__init__()
        self.fc1 = nn.Linear(1, hidden)
        self.fc2 = nn.Linear(hidden, int(L))
        self.act = _act(act)
        self.nonneg = bool(nonneg)

    def forward(self, u):  # (T,1) -> (T,L)
        y = self.fc2(self.act(self.fc1(u)))
        if self.nonneg:
            y = F.relu(y)
        return y

class MLPLearnableFeatureMap(nn.Module):
    # phi(x) = (1/sqrt(M)) * [ f_l(w_i^T x) ]_{i,l}
    def __init__(self, M, L, hidden=64, act="relu", nonneg=True,
                 chunk=1_000_000, shared_channels=False,
                 ch_rms=False, ch_rms_target=0.1):
        super().__init__()
        self.M, self.L = int(M), int(L)
        self.scale = 1.0 / math.sqrt(self.M)
        self.chunk = int(chunk)
        self.shared = bool(shared_channels)
        self.ch_rms = bool(ch_rms)
        self.ch_rms_target = float(ch_rms_target)

        if self.shared:
            self.shared_mlp = ScalarMLP(self.L, hidden, act, nonneg)
            self.mlps = None
        else:
            self.mlps = nn.ModuleList(
                [ScalarMLP(1, hidden, act, nonneg) for _ in range(self.L)]
            )

    def forward(self, proj):  # (B,H,N,M) -> (B,H,N,M*L)
        B, H, N, M = proj.shape
        u = proj.reshape(-1, 1).float()  # (T,1)

        if self.shared:
            y = _chunked(self.shared_mlp, u, self.chunk)  # (T,L)
        else:
            ys = [_chunked(mlp, u, self.chunk) for mlp in self.mlps]
            y = torch.cat(ys, dim=1)  # (T,L)
        y = y.view(B, H, N, M, self.L)
        if self.ch_rms:
            eps = 1e-6
            rms = torch.sqrt(y.pow(2).mean((0, 1, 2, 3)) + eps)  # (L,)
            s = (self.ch_rms_target / (rms + eps)).clamp(max=1.0)
            y = y * s.view(1, 1, 1, 1, self.L)
        return (y * self.scale).reshape(B, H, N, M * self.L)

\end{lstlisting}

\end{document}